\documentclass{MarcusArticle}
\usepackage[utf8]{inputenc}
\usepackage{censor}
\title{Advanced Stationary and Non-Stationary Kernel Designs for Domain-Aware Gaussian Processes}

\author[1,*]{Marcus M. Noack}
\author[1,2]{James A. Sethian} 
\affil[1]{The Center for Advanced Mathematics for Energy Research 
Applications (CAMERA), Lawrence Berkeley National Laboratory, Berkeley, CA 94720}
\affil[2]{Department of Mathematics, University of California, Berkeley}
\affil[*]{MarcusNoack@lbl.gov}

\date{\today}

\begin{document}
\maketitle
\begin{abstract}
    Gaussian process regression is a widely-applied method for function approximation and uncertainty quantification. 
    The technique has gained popularity recently in the machine learning community due to its robustness and interpretability.
    The mathematical methods we discuss in this paper are an extension of the Gaussian-process 
    framework. We are proposing advanced kernel designs that only allow for functions 
    with certain desirable characteristics to be elements of the reproducing kernel Hilbert space (RKHS) that underlies all kernel methods
    and serves as the sample space for Gaussian process regression.
    These desirable characteristics reflect the underlying physics; two obvious examples are symmetry and periodicity constraints. 
    In addition, we want to draw attention to non-stationary kernel designs that can be defined in the same framework to yield
    flexible multi-task Gaussian processes. 
    We will show the impact of advanced kernel designs on Gaussian processes 
    using several synthetic and two scientific data sets. The results show that informing a Gaussian process of 
    domain knowledge, combined with additional flexibility, 
    and communicated through advanced kernel designs, has a significant impact on the accuracy and relevance of the function approximation.
\end{abstract}
\section{Introduction}
Gaussian processes (GPs) \cite{williams2006gaussian} 
provide a powerful mathematical framework for function approximation from data.
The associated technique is generally referred to as Gaussian process regression (GPR).
GPs are flexible, robust, non-parametric and naturally include uncertainty quantification. Given some data
$\mathcal{D}~=~\{\mathbf{x}_i,y_i\}$, the GP regression model assumes
$y(\mathbf{x})~=~f(\mathbf{x})+\epsilon(\mathbf{x})$.
Here, $\mathbf{x}$ is the position in some input or parameter space, $y$ is the associated noisy function evaluation, and $\epsilon(\mathbf{x})$ represents the noise term. 
The covariance matrix $\mathbf{\Sigma}$ of 
the prior Gaussian probability distribution 
is defined via kernel functions $k(\mathbf{x}_i,\mathbf{x}_j; \phi)$, where $\phi$ is a set of hyperparameters that are
commonly found by maximizing the marginal log-likelihood of the data.
Kernels induce an inner product in a Hilbert space and therefore impose a metric, which can be interpreted as a similarity measure.
The flexibility of kernel functions and therefore of the associated similarity measure is the main focus of this paper.

\vspace{2mm}
\noindent
Gaussian processes have been shown to be applicable when it comes to domain-aware approximations of model functions. By placing one or more prior Gaussian probability distributions over a carefully-defined function space, and using the posterior Gaussian distribution in a way that captures the desired features, we can take into account several data sets and a variety of domain knowledge bases. Generally speaking, the theory of Gaussian processes allows for four main possibilities to communicate domain knowledge:
\begin{enumerate}
\item
We can extract subspaces of the function space in such a way that all elements have certain desired characteristics. The extracted function space is the, so-called, reproducing kernel Hilbert space (RKHS). This can be accomplished by developing advanced designs for stationary and non-stationary kernels;
\item 
The prior can be placed and shaped in accordance with any domain knowledge that allows us to define a separate prior over the same function space; this prior can also depend on the hyperparameters. This can, for instance, be accomplished by a constrained log-likelihood optimization or by minimizing the Kullback-Leibler divergence between priors \cite{swiler2020survey, matschek2019constrained,wang2016estimating}; 
\item The acquisition function, which acts on the posterior, can be useful when searching for certain features of the posterior mean and covariance functions. For instance can high (or low) gradients or curvatures in the posterior be highlighted \cite{noack2020advances};
\item  Flexible multi-task Gaussian processes can be defined using non-stationary kernels. Kernels can be seen as a similarity measure; the more flexibility they have, the more can be learned across the input space and the different tasks. The advantage here is that, if enough flexibility is provided to the kernel, no other changes have to made to a single-task GP.
\end{enumerate}

\vspace{2mm}
\noindent
In this work, we focus on points 1 and 4, which focus on advanced kernel designs for stationary and non-stationary kernels. Kernels dictate which functions are part of the RKHS, and are therefore optimally suited to impose hard constraints on the posterior mean. One interesting example of high importance is the assumption of symmetry or periodicity of the posterior mean. One goal of this paper is to show that taking advantage of advanced kernel designs makes GPR significantly more accurate. Designing these kernels can be done by taking advantage of permitted operators, such as adding kernels and applying linear operators to them. Advanced kernel designs also allow for a very natural way of dealing with multi-modal data sets. In the GP literature, this is often referred to as multi-task, multi-output or multivariate regression problems. Since there is no natural distance between different tasks, many workarounds have been proposed \cite{borchani2015survey,yu2005learning}. However, no workaround is necessary if the kernels are given enough flexibility to find the optimal distances between tasks. That way, most common problems of multi-output GPR, such as missing data or missing cross-task covariances, are addressed or avoided.

\vspace{2mm}
\noindent
{\bf{Contributions.}} The contributions in this paper can be summarized as follows: (1) We show how to tailor kernel designs to communicate domain-knowledge to the GP, using both known stationary kernels as well as introducing and deriving new stationary kernels; (2) We show how to build customized non-stationary kernels, using previously-published but not well-established non-stationary kernel designs; and (3) We draw attention to a natural way to implement multi-task GPs by formulating them in terms of non-stationary kernel designs. As we will see, while this idea is not new, it deserves reevaluation in the presence of advanced kernel designs. 

\vspace{2mm}
\noindent
{\bf{Organization.}} This paper is organized as follows. First we will show the basic Gaussian process regression framework which takes advantage of the standard kernel classes. We will see that, while the standard kernels are very general, there are weaknesses associated with them, which lead to unnecessary inaccuracies of the approximation. Second, we will show the mathematics that is needed to make a Gaussian process domain aware by defining advanced kernel designs. These designs are partly known to the Gaussian-process community but largely unknown to the practitioner. While presenting them, we will show their impact on GP function approximations directly. Here we will also discuss kernels for multi-task Gaussian processes. Third, we will show the impact of the presented methodologies on experimental data which will be simulated by using previously-acquired scientific data sets.
\section{The Mathematics of Advanced Kernel Designs for Gaussian Processes} 
\subsection{Preliminaries}\label{sec:pre}
We define a set $\mathcal{X}_i \subset \mathbb{R}^{n_1}$, which is often referred to as the parameter space or the input space, and elements $\mathbf{x}_i~\in \mathcal{X}_i$.
We also define a set $\mathcal{X}_o \subset \mathbb{R}^{n_2}$ with elements
$\mathbf{x}_o~\in \mathcal{X}_o$, which represent the arbitrary but fixed indices of all function values of a vector-valued function whose domain consists of values that are elements of $\mathcal{X}_i$. 
These indices often have no physical meaning or equivalent,
and have to be chosen arbitrarily, which constitutes the main difficulty for multi-task regression. 
To tackle multi-output Gaussian process regression, we are defining the Cartesian
product space $\mathcal{X}~=~\mathcal{X}_i~\times~\mathcal{X}_o,~\mathcal{X}~\subset~\mathbb{R}^{n_1+n_2}$
with elements $\mathbf{x}=[\mathbf{x}_i,\mathbf{x}_o]^T$. 
We call this set the index set, because the functions defined on it are elements of 
a function space we will define later. Note however, that $\mathcal{X}_i$ as well as $\mathcal{X}_o$ are considered index sets, since their elements index function values of functions that are themselves elements of a set. The assumption that the input and output spaces are a subspace of the Euclidean space is not a strict requirement but used here for simplicity.
We define a total of five functions on $\mathcal{X}$.
First, the latent function $f=f(\mathbf{x})$ which can be interpreted as the
inaccessible ground truth. 
Second, the often noisy measurements 
$y=y(\mathbf{x}):~\mathcal{X}~\rightarrow~\mathbb{R}$.
Third, the surrogate model function, which is defined
as $\rho=\rho(\mathbf{x}):~\mathcal{X}~\rightarrow~\mathbb{R}$. 
Fourth, the posterior mean function $m(\mathbf{x})$.
Fifth, the posterior variance function $\sigma^2(\mathbf{x})$.
We note that typically in multi-task GPR, those functions are vector-valued functions.
Since we are introducing the Cartesian product space $\mathcal{X}$, this is not necessary;
we have effectively reduced a multi-output Gaussian process to a single-output Gaussian process.
Since we are not interpreting the tasks as a set of functions on $\mathcal{X}_i$ but instead as a scalar function on $\mathcal{X}_o$, we refer to 
this as a function-valued GP. For instance, the output could be defined on $\mathbb{R}$, which leads to $\mathcal{X}=\mathcal{X}_i~\times~\mathbb{R}$; in this case the output is a function on $\mathbb{R}$.
The general concept of transforming a multi-output GP to a single-output GP is not new, and is normally referred to as single-target method or output-as-input-view \cite{van2020framework} and criticized for not taking into account cross-task covariances \cite{borchani2015survey}; however, this criticism only applies to stationary kernels. One of the goals of this paper is to achieve 
cross-task covariances by defining flexible non-stationary kernels.

\vspace*{2mm}
\noindent
Next, we define a pre-Hilbert space
\begin{align}
    \mathcal{H}=\{f(\mathbf{x}):f(\mathbf{x})=\sum_i^N \alpha_i k(\mathbf{x}_i,\mathbf{x}), \forall \pmb{\alpha}\in \mathbb{R}^N, \mathbf{x}\in \mathbb{R}^n\},
    \label{eq:RKHS}
\end{align}
with covariance function $k(\mathbf{x}_i,\mathbf{x})$; $N$ is the number of data point locations, and 
$f(\mathbf{x})$ is the unknown latent function. Strictly speaking, Equation \eqref{eq:RKHS} is not a full infinite-dimensional pre-Hilbert space but a finite-dimensional sub-space spanned by the data. For it to qualify as a Hilbert space we have to equip the space with the norm $||f||_{\mathcal{H}}=\sqrt{\langle f,f\rangle_{\mathcal{H}}}$ and add all limit points of sequences that converge in that norm. As a reminder, note that scalar functions over $\mathcal{X}$, e.g. $f(\mathbf{x})$,
are vectors (bold typeface) in $\mathcal{H}$.
A kernel induces an inner product of two elements $\in \mathcal{H}$, i.e,\\ $\langle f_1(x),f_2(x)\rangle_{\mathcal{H}}~=~\sum_i \sum_j \alpha_i \beta_j k(\mathbf{x}_i,\mathbf{x}_j)$, where $\alpha_i$ and $\beta_j$ are coefficients.
\begin{definition}\label{def:kernel}
A kernel is a symmetric and positive semi-definite (p.s.d.) function $k(\mathbf{x_1},\mathbf{x_2})$,
$\mathcal{X}~\times~\mathcal{X}\rightarrow~\mathbb{R}$,\\
it therefore satisfies
$\sum_i^N\sum_j^N~ c_i~c_j~k(x_i,x_j)~
\geq~0~\forall N,~\mathbf{x}\in~\mathcal{X}, ~\mathbf{c}\in~\mathbb{R}^N$
\end{definition}

\vspace*{2mm}
\noindent
Given that definition, it is clear that the set of kernels is closed under addition, multiplication
and linear transformation \cite{ginsbourger2013kernels}, which we will build upon later on.
Given the definition of the pre-Hilbert space (Equation \ref{eq:RKHS}), it can be shown that for elements of
$\mathcal{H}$
\begin{align}
    \langle k(\mathbf{x}_0,\mathbf{x}),f(\mathbf{x}) \rangle~=~f(\mathbf{x}_0),
    \label{eq:reproducing}
\end{align}
which is the reason the completion of the space $\mathcal{H}$ is called Reproducing Kernel Hilbert Space. 
Reproducing in that context refers to the fact that the inner product of a kernel, evaluated at a point, with a function produces the function at that point; the inner product ``reproduces'' the function value, which is the essence of Equation \eqref{eq:reproducing}.
Gaussian processes are based on defining a prior probability distribution over the RKHS. 
In this case the kernels are understood as covariances
\begin{align}
    k(\mathbf{x}_1,\mathbf{x}_2)~=~\int_\mathcal{H}~f(\mathbf{x}_1)f(\mathbf{x}_2)q(f)~df,
    \label{eq:cov_kernel}
\end{align}
where $q$ is some density function.

\subsection{A Birds-Eye View on Gaussian Processes}
Given data $\mathcal{D}~=~\{\mathbf{x}_i,y_i\}$, a prior probability distributions over functions $f(\mathbf{x})$ can be defined as
\begin{equation}
    p(\mathbf{f})=\frac{1}{\sqrt{(2\pi)^\mathrm{dim}|\mathbf{K}|}}
    \exp \left[ -\frac{1}{2}(\mathbf{f}-\pmb{\mu})^T \mathbf{K}^{-1}(\mathbf{f}-\pmb{\mu}) \right],
    \label{eq:prior}
\end{equation}
where $\mathbf{K}$ is the covariance matrix of the data, calculated by applying the covariance kernel
$k(\mathbf{x}_1,\mathbf{x}_2)$ (see Definition \ref{def:kernel}) to
the data positions, and $\pmb{\mu}$ is the prior mean vector. 
We can define the likelihood over functions $y(\mathbf{x})$ as
\begin{equation}
    p(\mathbf{y},\mathbf{f})=\frac{1}{\sqrt{(2\pi)^\mathrm{dim}|\mathbf{V}|}}
    \exp \left[ -\frac{1}{2}(\mathbf{y}-\mathbf{f})^T \mathbf{V}^{-1} (\mathbf{y}-\mathbf{f}) \right],
    \label{eq:likelihood_iid}
\end{equation}
where $\mathbf{V}$ is the matrix of the non-i.i.d.~noise \cite{noack2020autonomous}.
The noise is responsible for the difference between the unknown latent function $f(\mathbf{x})$ and the measurements $y(\mathbf{x})$.
In the standard literature, often only i.i.d.~noise is discussed,
which is insufficient for many applications \cite{noack2020autonomous}. 

\vspace*{2mm}
\noindent
The vast majority of work published about Gaussian processes uses only a handful of
standard kernels to compute covariances. By far the most frequently used kernel is the squared exponential kernel
\begin{equation}
    k(\mathbf{x}_1,\mathbf{x}_2)~=~\sigma_s^2\exp{[-\frac{||\mathbf{x}_1-\mathbf{x}_2||^2}{2l^2}]},
\end{equation}
where $\sigma_s^2$ is the signal variance and $l$ is the length scale 
which can be anisotropic, as we will see later.
The signal variance and the length scales are examples of the so-called hyperparameters of the Gaussian process
and are calculated by solving
\begin{align}
    \argmax_{\sigma_s^2,\phi} \Big(
    &\log(L(D;\sigma_s^2,\phi))~=\nonumber \\ 
    &~-\frac{1}{2}(\mathbf{y}-\pmb{\mu}(\phi))(\mathbf{K}(\phi)+\mathbf{V})^{-1}(\mathbf{y}-\pmb{\mu}(\phi)) \nonumber \\ 
    &~-\frac{1}{2}\log(|\mathbf{K}(\phi)+\mathbf{V}|)
                         -\frac{\dim(\mathbf{y})}{2}\log(2\pi)\Big).\label{eq:likelihood}                     
\end{align}
Given the hyperparameters, we can calculate and condition the joint prior 
\begin{equation}
    p(\mathbf{f},\mathbf{f}_0)=\frac{1}{\sqrt{(2\pi)^\mathrm{dim}|\mathbf{\Sigma}|}}
    \exp \left[ -\frac{1}{2} \Big (
    \begin{bmatrix} \mathbf{f}-\pmb{\mu} \\  \mathbf{f}_0-\pmb{\mu}_0 \end{bmatrix}^T 
    \mathbf{\Sigma}^{-1}
    \begin{bmatrix} \mathbf{f}-\pmb{\mu} \\  \mathbf{f}_0-\pmb{\mu}_0 \end{bmatrix} \Big )\right ] ,
    \label{eq:joint_prior}
\end{equation}
where 
\begin{align}
   \Sigma~=~
   \begin{pmatrix}
   \mathbf{K}     & \pmb{\kappa} \\
   \pmb{\kappa}^T & \pmb{\mathcal{K}}
   \end{pmatrix},
   \label{eq:joint_covariance_matrix}
\end{align}
to obtain the well-known
posterior
\begin{align}\label{eq:pred_distr}
    p(\mathbf{f}_0|\mathbf{y})~&=~\int_{\mathbb{R^N}} 
    p(\mathbf{f}_0|\mathbf{f},\mathbf{y})~p(\mathbf{f},\mathbf{y})~d\mathbf{f} \nonumber \\
    &~\propto \mathcal{N}(\pmb{\mu}+\pmb{\kappa}^T~
    (\mathbf{K}+\mathbf{V})^{-1}~(\mathbf{y}-\pmb{\mu}), \pmb{\mathcal{K}} -
    \pmb{\kappa}^T~(\mathbf{K}+\mathbf{V})^{-1}~\pmb{\kappa}),
\end{align}
where $\pmb{\kappa}_i~=~k(\mathbf{x}_0,\mathbf{x}_i,\phi)$, 
$\pmb{\mathcal{K}}~=~k(\mathbf{x}_0,\mathbf{x}_0,\phi)$ and $\mathbf{K}_{ij}~=~k(\mathbf{x}_i,\mathbf{x}_j,\phi)$. $\mathbf{x}_0$ are the points at which
the Gaussian posterior should be predicted. $\mathbf{f}_0$ are values of the latent function $f$ at the points $\mathbf{x}_0$. The posterior contains the posterior mean $m(\mathbf{x})$
and the posterior variance $\sigma^2(\mathbf{x})$ (see also Section \ref{sec:pre}).
\subsection{Basic Kernel Design and its Weaknesses}
In standard Gaussian process regression, we assume one function value (single task) to be approximated
on the index set. 
Distances within the index set are generally assumed to be isotropic and Euclidean.
Moreover, we often assume first and second order stationarity of the process. These assumptions translate into kernels of the form 
\begin{align}
    k(\mathbf{x}_1,\mathbf{x}_2)~=~k(||\mathbf{x}_1 - \mathbf{x}_2||,\sigma_s^2,l),
\end{align}
where $l$ is the isotropic and constant length scale and $\sigma_s^2$ is the constant signal variance, and constant prior-mean functions.
Therefore, independent of the dimensionality of the index set, we only have 
to solve Equation \eqref{eq:likelihood}
for two hyperparameters, one signal variance and one length scale. 
In addition, we do not assume any particular characteristics of the function to be explored, which translates
to the use of standard kernels (e.g. Mat\'ern, exponential, squared exponential). 
In the vast majority of published work, 
the squared exponential kernel is used \cite{pilario2020review}. 
Also, when several tasks are involved, they are often assumed to be independent in the standard GP framework.
If the tasks are assumed to be correlated, the used methods are based on significant augmentations of the basic GP theory \cite{borchani2015survey}.

\vspace{2mm}
\noindent
While the standard approach yields an agnostic and widely applicable approach to 
regression,
it also has some considerable drawbacks. 
In this paper, we focus on two of them: 

\begin{itemize}
    \item 
First, by defining any kernel, the user implicitly chooses which functions --- carrying hidden restrictions or assumptions --- are elements of the RKHS. 
Often, this is done by accident without knowing what is imposed.
For instance, using the squared exponential kernel imposes an infinite order of 
differentiability onto the posterior mean functions, even if this 
assumption is actually not backed by physics or other domain knowledge.
Limiting the approximated function to certain characteristics that are not reasonable should be avoided.
An alternative, which is one focus of this paper and discussed below, exploits the fact that the user often knows certain local and global characteristics of the posterior mean. 
Enforcing them can yield a vastly improved accuracy of the approximation. 
In the case of stationary kernel designs, this is due to extra information that is 
propagated to unexplored regions of the index set. 
\item
Second, stationary kernels do not have the flexibility to encode varying similarities across
the index set. That means, two points in one corner of the domain will have the same inner product as two
other points in the other corner of the domain, as long as distances between the points are the same; 
the inner product and therefore the similarity will not depend on the respective location of the point pairs. 
This makes it difficult to learn similarities across tasks, since there is no natural distance
between tasks. One alternative we want to draw the reader's attention to, 
is to addresses this issue by adding extra flexibility to non-stationary 
kernels which translates into a method that is able to learn more complicated
patterns of the data set across the input and the output space (tasks). 
In doing so, and in contrast to standard methods, the basic theory of GPs remains unchanged, 
and the entire difference between single-task and multi-task GPs is contained within the kernel design, maintaining the inherent robustness of a GP and avoiding common problems such as missing data, assumed linear dependence of tasks, and reduced interpretability. 
\end{itemize}
We want to note here that there are many methods to address the issues with 
multi-task Gaussian processes. See \citet{borchani2015survey} for a comprehensive overview.

\vspace{2mm}
\noindent
To reiterate, the main goal in this work is to define positive semi-definite functions
\begin{align}
    k:\mathcal{X}~\times~\mathcal{X}~\rightarrow \mathbb{R}
\end{align}
that serve as stationary and non-stationary kernel functions and 
effectively extract a RKHS in such a way that it only contains functions with certain desirable characteristics
and is able to encode and learn complicated cross-task covariances.

\vspace{2mm}
\noindent
We next discuss stationary kernels, showing how to incorporate hard constraints on the posterior mean. This is followed by a generalization to non-stationary kernels, showing how they can provide a framework for multi-task GPs. 
\subsection{Advanced Stationary Kernel Designs for Hard Constraints on the Posterior Mean}
Stationary kernels are positive definite functions
of the form 
\begin{align}
    k(\mathbf{x}_1,\mathbf{x}_2)~=~k(||\mathbf{x}_1-\mathbf{x}_2||),
\end{align}
where $||\cdot||$ is some norm. The Euclidean norm is used in the overwhelming majority of studies.

\vspace*{2mm}
\noindent
The set of kernel functions is closed under addition, multiplication, and application of linear operators.
Therefore, kernel functions can be combined in many ways to formulate powerful definitions of similarity
between data points. We will see that this can be used to inform the process that similarity is recurrent in
$\mathcal{X}$ or follows a certain structure.
\subsubsection{Stationary Kernels Constraining Differentiability}
The Mat\'ern kernel class is defined as
\begin{align}
       k(\mathbf{x}_1,\mathbf{x}_2)=\frac{1}{2^{v-1}\Gamma(v)}
\Big( \frac{\sqrt{2v}}{l}r\Big)^v B_v \Big( \frac{\sqrt{2v}}{l}r\Big),
\end{align}
where $r$ is some metric in $\mathcal{X}$, $B_v$ is the modified Bessel function and $v$ is the parameter
controlling the differentiability. Combined with anisotropic kernel definitions, a practitioner can control the level of differentiability in each direction of the input space. 
This is a rather well-known characteristic of kernels and is included here for completeness. 
\subsubsection{Kernels for Additive Functions}
Approximating a function of the form $\sum_i~g_i(x_i)$ can be accomplished by choosing a  Gaussian process with defining kernel
\begin{align}
    k(\mathbf{x}_1,\mathbf{x}_2)~=~\sum_i~k_i(\mathbf{x}_1^i,\mathbf{x}_2^i).
\end{align}
The resulting process can propagate information about the function into regions of the index set where information is given only in ($n-1$-dimensional) subspaces. Figure \ref{fig:additive} shows an example of how additive kernels can be used.
While points are only given in a sub-space of the index set, the information can be propagated into all of $\mathcal{X}$. 
The standard kernel used in the example shown in Figure \ref{fig:additive} is defined as
\begin{equation}\label{eq:st_kernel}
    k(\mathbf{x}_1,\mathbf{x}_2)~=~2~\exp{[-\frac{|\mathbf{x}_1-\mathbf{x}_2|}{0.5}]},
\end{equation}
where $|\cdot|$ denotes the Euclidean distance in $\mathcal{X}$.
The additive kernel is defined as
\begin{equation}\label{eq:add_kernel}
    k(\mathbf{x}_1,\mathbf{x}_2)~=~\exp{[-\frac{|x_1^1-x_2^1|}{0.5}]} + \exp{[-\frac{|x_1^2-x_2^2|}{0.5}]}.
\end{equation}
Figure \ref{fig:additive} shows how powerful the knowledge of additivity can be; information can be propagated far away from the available data along axes directions. 
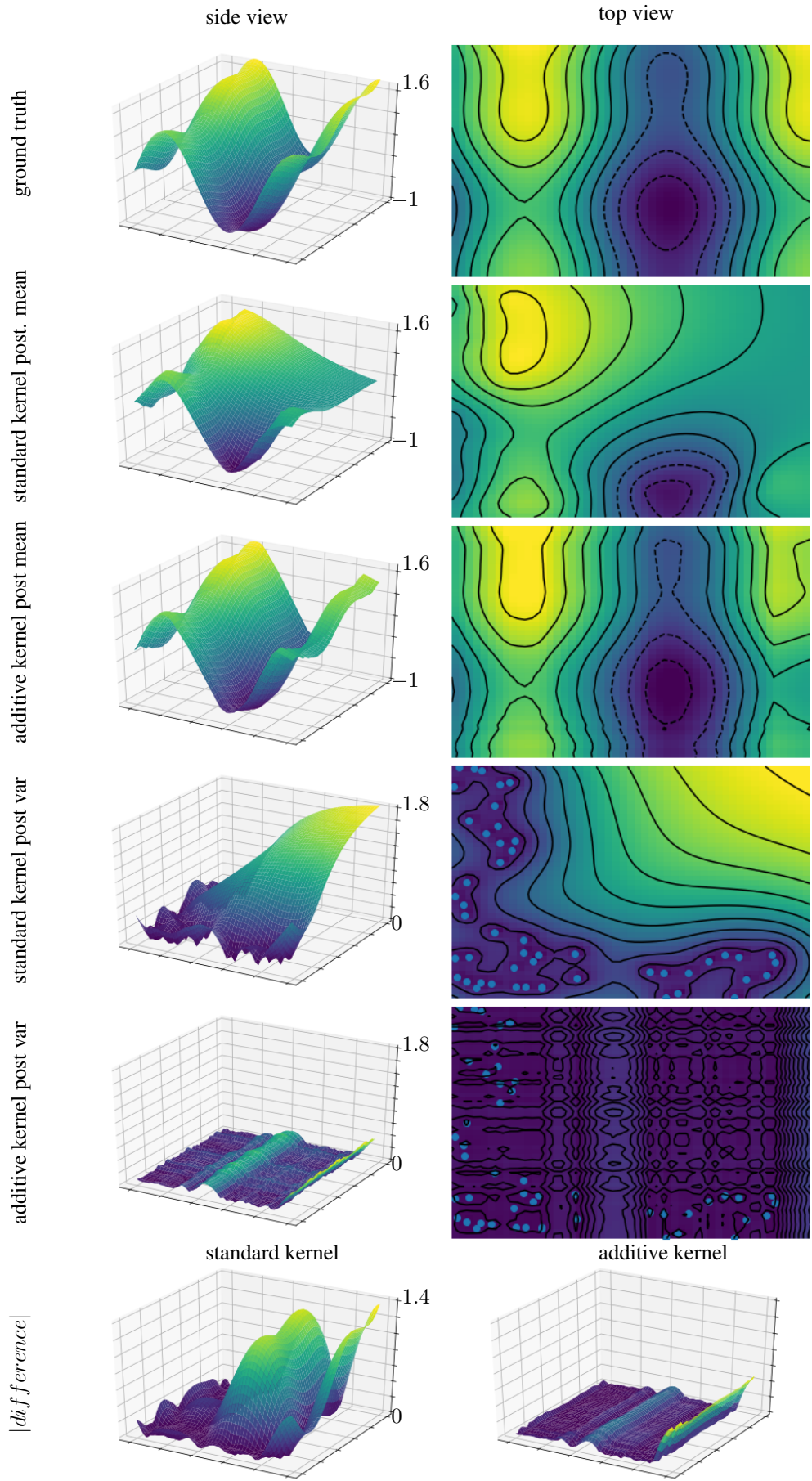
\begin{figure}
    \centering
    \begin{subfigure}[t]{0.40\linewidth}
    \includegraphics[width = \linewidth]{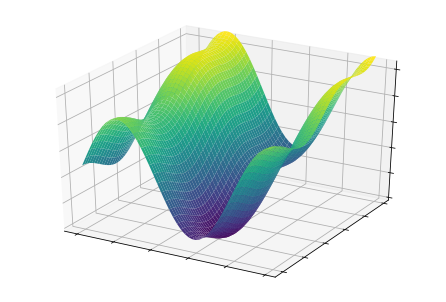}
    \end{subfigure}    
    \begin{subfigure}[t]{0.40\linewidth}
    \includegraphics[width = \linewidth]{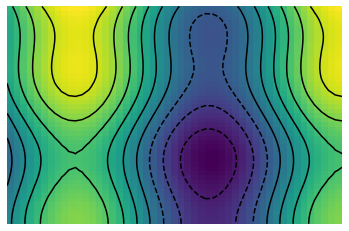}
    \put(-360,40){\rotatebox{90}{\makebox(0,0)[lb]{ground truth}}}
    \put(-190,90){{\makebox(0,0)[lb]{$1.6$}}}
    \put(-195,35){{\makebox(0,0)[lb]{$-1$}}}
    \put(-280,120){{\makebox(0,0)[lb]{side view}}}
    \put(-100,120){{\makebox(0,0)[lb]{top view}}}
    \end{subfigure}
    
    \vspace{-2mm}
    \begin{subfigure}[t]{0.40\linewidth}
    \includegraphics[width = \linewidth]{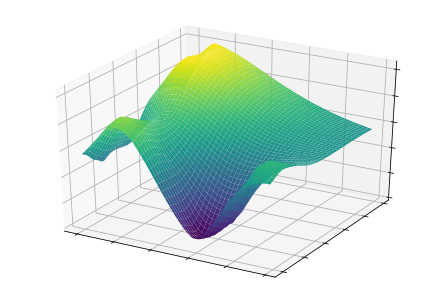}
    \end{subfigure}    
    \begin{subfigure}[t]{0.40\linewidth}
    \includegraphics[width = \linewidth]{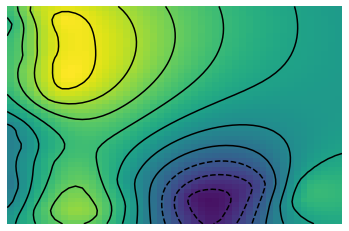}
    \put(-360,10){\rotatebox{90}{\makebox(0,0)[lb]{standard kernel post. mean}}}
    \put(-190,90){{\makebox(0,0)[lb]{$1.6$}}}
    \put(-195,35){{\makebox(0,0)[lb]{$-1$}}}
    \end{subfigure}
    
    \vspace{-2mm}
    \begin{subfigure}[t]{0.40\linewidth}
    \includegraphics[width = \linewidth]{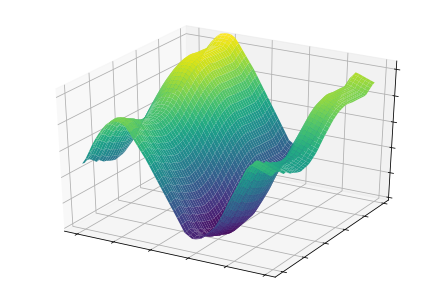}
    \end{subfigure}    
    \begin{subfigure}[t]{0.40\linewidth}
    \includegraphics[width = \linewidth]{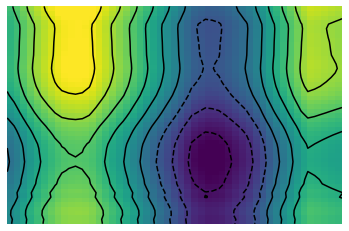}
    \put(-360,10){\rotatebox{90}{\makebox(0,0)[lb]{additive kernel post mean}}}
    \put(-190,90){{\makebox(0,0)[lb]{$1.6$}}}
    \put(-195,35){{\makebox(0,0)[lb]{$-1$}}}
    \end{subfigure}
    
    \vspace{-2mm}
    \begin{subfigure}[t]{0.40\linewidth}
    \includegraphics[width = \linewidth]{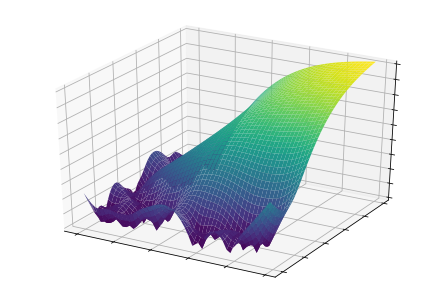}
    \end{subfigure}    
    \begin{subfigure}[t]{0.40\linewidth}
    \includegraphics[width = \linewidth]{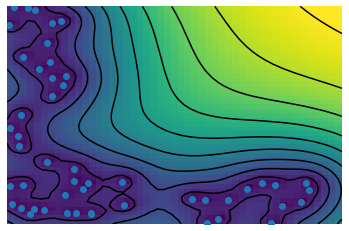}
    \put(-360,10){\rotatebox{90}{\makebox(0,0)[lb]{standard kernel post var}}}
    \put(-190,90){{\makebox(0,0)[lb]{$1.8$}}}
    \put(-195,35){{\makebox(0,0)[lb]{$0$}}}
    \end{subfigure}
    
    \vspace{-2mm}
    \begin{subfigure}[t]{0.40\linewidth}
    \includegraphics[width = \linewidth]{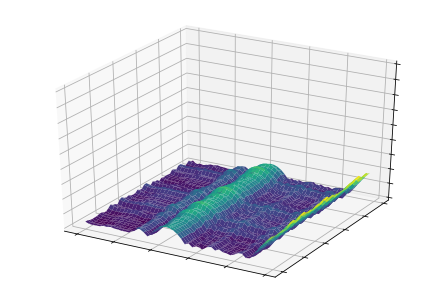}
    \end{subfigure}    
    \begin{subfigure}[t]{0.40\linewidth}
    \includegraphics[width = \linewidth]{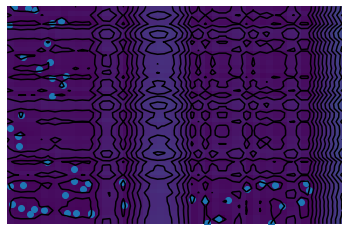}
    \put(-360,10){\rotatebox{90}{\makebox(0,0)[lb]{additive kernel post var}}}
    \put(-190,90){{\makebox(0,0)[lb]{$1.8$}}}
    \put(-195,35){{\makebox(0,0)[lb]{$0$}}}
    \end{subfigure}
    
    \begin{subfigure}[t]{0.40\linewidth}
    \includegraphics[width = \linewidth]{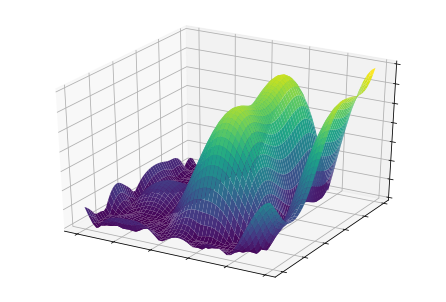}
    \end{subfigure}    
    \begin{subfigure}[t]{0.40\linewidth}
    \includegraphics[width = \linewidth]{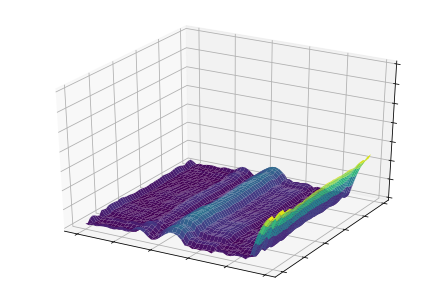}
    \put(-360,30){\rotatebox{90}{\makebox(0,0)[lb]{$|difference|$}}}
    \put(-280,110){{\makebox(0,0)[lb]{standard kernel}}}
    \put(-100,110){{\makebox(0,0)[lb]{additive kernel}}}
    \put(-190,90){{\makebox(0,0)[lb]{$1.4$}}}
    \put(-195,35){{\makebox(0,0)[lb]{$0$}}}
    \end{subfigure}
    \caption{Standard kernel (Equation \eqref{eq:st_kernel}) vs additive kernel (Equation \eqref{eq:add_kernel}) function to approximate a function on $\mathcal{X}=[0,1]\times[0,1]$. Using the additive kernel means propagating information into regions where no data is available. The estimated variances are significantly smaller compared to the use of standard kernels.}
    \label{fig:additive}
\end{figure}
\subsubsection{Anisotropy of Distance Measures on $\mathcal{X}$}
In addition to summation, one can also combine kernels by a product.
In both cases, every direction
can have its own length scale, giving rise to one formulation of anisotropy
\begin{align}
    k(\mathbf{x}_1,\mathbf{x}_2)~&=~\exp{[-\frac{|x_1^1-x_2^1|}{l_1}]} + \exp{[-\frac{|x_1^2-x_2^2|}{l_2}]} \hspace{.2in} [Additive] \label{add_ani} \\
    k(\mathbf{x}_1,\mathbf{x}_2)~&=~\exp{[-\frac{|x_1^1-x_2^1|}{l_1}]} ~ \exp{[-\frac{|x_1^2-x_2^2|}{l_2}]}
    \hspace{.2in}[Multiplicative] \label{mult_ani}.
\end{align}
However, the additive kernel comes with additional properties presented in the last section. If the model function is not additive,
the use of kernel \eqref{add_ani} will lead to wrong predictions.
Another way of implementing anisotropy is by altering the Euclidean distance in $\mathcal{X}$ with a different metric such that
\begin{align}
    k(\mathbf{x}_1,\mathbf{x}_2) = k((\mathbf{x}_1-\mathbf{x}_2)^T~\mathbf{M}~(\mathbf{x}_1-\mathbf{x}_2)),
\end{align}
where $M$ is any symmetric positive definite matrix.
More on that can be found in \cite{williams2006gaussian,noack2020autonomous}.
Anisotropy plays an important role in many data sets and its inclusion is vitally important  \cite{noack2020autonomous}.
\subsubsection{Linear Operators Acting on Kernels}
\citet{ginsbourger2013kernels} pointed out that kernels can be passed through linear operators.
\begin{theorem}
If $k(x_1,x_2)$ is a kernel, then $L_{x_1}(L_{x_2}(k))$ is also a valid kernel function.
\end{theorem}

\vspace*{2mm}
\noindent
We can use this theorem to derive kernels for many different situations.
Examples include enforcing axial or rotational symmetry, or periodicity upon the model function.
To showcase the procedure, axial symmetry in two dimensions can be enforced by applying the operator
\begin{align}
    L(f(\mathbf{x}))~=~\frac{f([x^1,x^2]^T)+f([-x^1,x^2]^T)+f([x^1,-x^2]^T)+f([-x^1,-x^2]^T)}{4}
\end{align}
on a kernel function, which results in
\begin{align}\label{eq:symm_kernel}
    L_{\mathbf{x}_1}(k(\mathbf{x}_1,\mathbf{x}_2))~=~1/4~(&k(\mathbf{x}_1,\mathbf{x}_2)+k([-x^1_1, x^2_1]^T,\mathbf{x}_2) \nonumber \\
                                                           +&k([x^1_1, -x^2_1]^T,\mathbf{x}_2)+k([-x^1_1, -x^2_1]^T,\mathbf{x}_2)) \nonumber \\
    L_{\mathbf{x}_2}(k(\mathbf{x}_1,\mathbf{x}_2))~=~1/4~(&k(\mathbf{x}_1,\mathbf{x}_2)+k(\mathbf{x}_1,[-x^1_2, x^2_2]^T) \nonumber \\
                                                           +&k(\mathbf{x}_1,[x^1_2, -x^2_2]^T)+k(\mathbf{x}_1,[-x^1_2, -x^2_2]^T)) \nonumber \\
                                                            \Rightarrow & \nonumber \\
    L_{\mathbf{x}_2}(L_{\mathbf{x}_1}(k(\mathbf{x}_1,\mathbf{x}_2))) = 1/16~(&k(\mathbf{x}_1,\mathbf{x}_2)+k([-x^1_1, x^2_1]^T,\mathbf{x}_2) \nonumber \\
                                                                              +&k([x^1_1, -x^2_1]^T,\mathbf{x}_2)+k([-x^1_1, -x^2_1]^T,\mathbf{x}_2) \nonumber \\
                                                                              +&k(\mathbf{x}_1,[-x^1_2, x^2_2]^T) \nonumber \\
                                                                              +&k(\mathbf{x}_1,[x^1_2, -x^2_2]^T)+k(\mathbf{x}_1,[-x^1_2, -x^2_2]^T)\nonumber \\
                                                                              +&k([-x^1_1, x^2_1]^T,[-x^1_2, x^2_2]^T)+k([-x^1_1, x^2_1]^T,[x^1_2, -x^2_2]^T)\nonumber \\
                                                                              +&k([-x^1_1, x^2_1]^T,[-x^1_2, -x^2_2]^T)+k([x^1_1, -x^2_1]^T,[-x^1_2, x^2_2]^T)\nonumber \\
                                                                              +&k([x^1_1, -x^2_1]^T,[x^1_2, -x^2_2]^T)+k([x^1_1, -x^2_1]^T,[-x^1_2, -x^2_2]^T)\nonumber \\
                                                                              +&k([-x^1_1,-x^2_1]^T,[-x^1_2, x^2_2]^T)+k([-x^1_1,-x^2_1]^T,[ x^1_2, -x^2_2]^T)\nonumber \\
                                                                              +&k([-x^1_1, -x^2_1]^T,[-x^1_2, -x^2_2]^T)),
\end{align}
where $k(\mathbf{x}_1,\mathbf{x}_1)$ can be any kernel, for instance the anisotropic squared exponential kernel
\begin{align}
    k(\mathbf{x}_1,\mathbf{x}_1)~=~\exp{[-\frac{\langle \mathbf{x}_1-\mathbf{x}_2,\mathbf{M},\mathbf{x}_1-\mathbf{x}_2\rangle }{l}]}.
\end{align}
See Figure \ref{fig:symmetric_ackleys} for a presentation of the effects of such a kernel.
\begin{figure}[t!]
    \centering
    \includegraphics[width =  \linewidth]{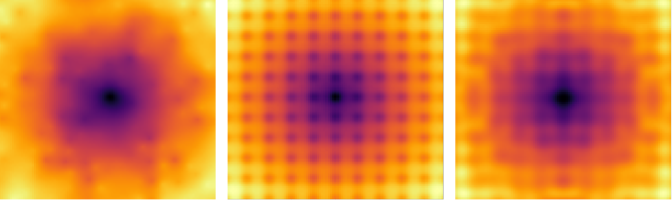}
    \put(-390,140){{\makebox(0,0)[lb]{standard kernel}}}
    \put(-240,140){{\makebox(0,0)[lb]{ground truth}}}
    \put(-105,140){{\makebox(0,0)[lb]{symmetric kernel}}}
    \caption{The powerful effect of kernel-based constraints on a GP regression. Displayed is Ackley's function (middle) and two GP posterior mean
    models. Left is the posterior mean calculated with an unconstrained Gaussian process. On the right is the posterior mean, calculated with imposed
    axial symmetry. In this example, the axial symmetry improves the uncertainty at a given number of measurements 4 fold and increases the computational speed 64 fold.}
    \label{fig:symmetric_ackleys}
\end{figure}
To inform the GP about periodicity in $x^2$ direction, we can define the linear operator
\begin{align}
    L(f(\mathbf{x}))~=~\frac{f([x^1 , x^2 ]^T)+
                              f([x^1 , x^2 + p ]^T)+
                              f([x^1 , x^2 - p ]^T)}{3},
\end{align}
from which the following kernel can be derived
\begin{align}
        L_{\mathbf{x}_2}&(L_{\mathbf{x}_1}(k(\mathbf{x}_1,\mathbf{x}_2))) = \nonumber \\
        1/9~(&k(\mathbf{x}_1,\mathbf{x}_2)+k(\mathbf{x}_1,[x^1_2, x^2_2 + p])^T +k(\mathbf{x}_1,[x^1_2, x^2_2 - p])^T \nonumber \\
        + & k([x^1_1, x^2_1 + p]^T,\mathbf{x}_2)+k([x^1_1, x^2_1 + p]^T,[x^1_2, x^2_2 + p]^T) +k([x^1_1, x^2_1 + p]^T,[x^1_2, x^2_2 - p]^T) \nonumber \\
        + & k([x^1_1, x^2_1 - p]^T,\mathbf{x}_2)+k([x^1_1, x^2_1 - p]^T,[x^1_2, x^2_2 + p]^T) +k([x^1_1, x^2_1 - p]^T,[x^1_2, x^2_2 - p]^T)),
\end{align}
where $p$ is the period.
Figure \ref{fig:periodic_kernel} presents a posterior mean model that results from such a kernel. 
\begin{figure}
    \centering
    \begin{subfigure}[t]{0.32\linewidth}
    \includegraphics[trim={1cm 0.8cm 0.9cm 0.8cm},clip,width = \linewidth]{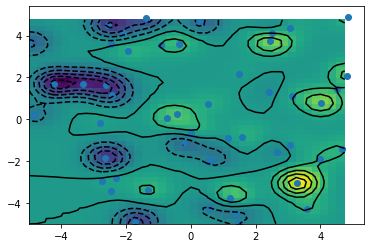}
    \end{subfigure}    
    \begin{subfigure}[t]{0.32\linewidth}
    \includegraphics[trim={1cm 0.8cm 0.0cm 0.0cm},clip,width = \linewidth]{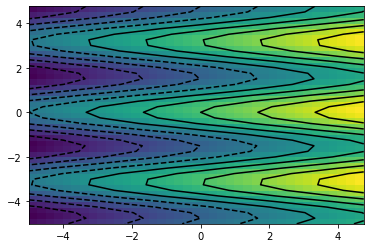}
    \end{subfigure}
    \begin{subfigure}[t]{0.32\linewidth}
    \includegraphics[trim={1cm 0.8cm 0.9cm 0.8cm},clip,width = \linewidth]{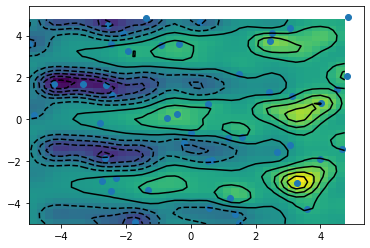}
    \end{subfigure}
    \caption{Posterior mean function given 50 data points for a standard GP on the left and a GP informing the posterior mean about periodicity on the right. The ground truth can be seen in the center. The blue points show the measurement locations. The periodicity from
    the kernel is enforced and used to inform the posterior mean in places without data.}
    \label{fig:periodic_kernel}
\end{figure}

\vspace*{2mm}
\noindent
One could argue that axial symmetry is more of academic than of practical interest, since the knowledge about axial symmetry can simply be accounted for by limiting the range for probing the function. However, the general principle can be extended to more complicated symmetries, such as rotational symmetry. For instance, the kernel for six-fold symmetry in two dimensions is defined as
\begin{equation}\label{eq:six-fold}
    k(\mathbf{x}_1,\mathbf{x}_2)=\frac{1}{36}\sum_{\phi \in p \pi/3}~ \sum_{\theta \in q \pi/3}~\tilde{k}(\mathcal{R}_{\phi}\mathbf{x}_1,\mathcal{R}_{\theta}\mathbf{x}_2);~p,q \in \{0,1,2,3,4,5\},
\end{equation}
where $\tilde{k}$ is any valid stationary kernel, and $\mathcal{R}$ is a rotation matrix rotating the vector $\mathbf{x}$ by a specified angle. The result of such a kernel can be seen in Figure \ref{fig:six-fold}.
\begin{figure}
    \centering
    \begin{subfigure}[t]{0.32\linewidth}
    \includegraphics[trim={1.2cm 0.8cm 0.0cm 0.0cm},clip,width = \linewidth]{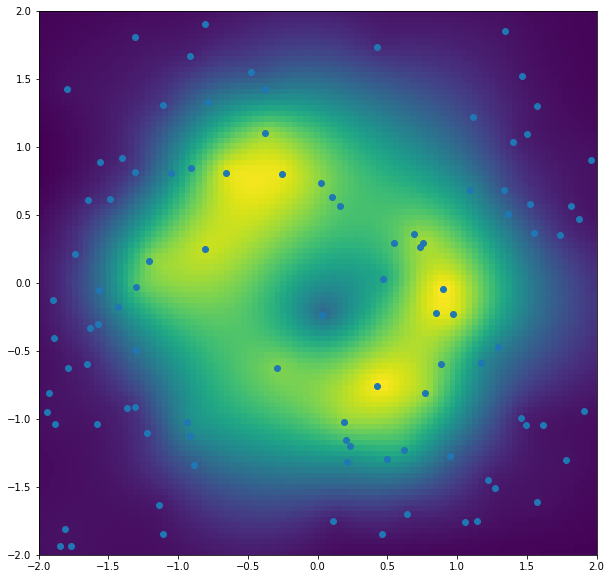}
    \end{subfigure}    
    \begin{subfigure}[t]{0.32\linewidth}
    \includegraphics[trim={1.2cm 0.8cm 0.0cm 0.0cm},clip,width = \linewidth]{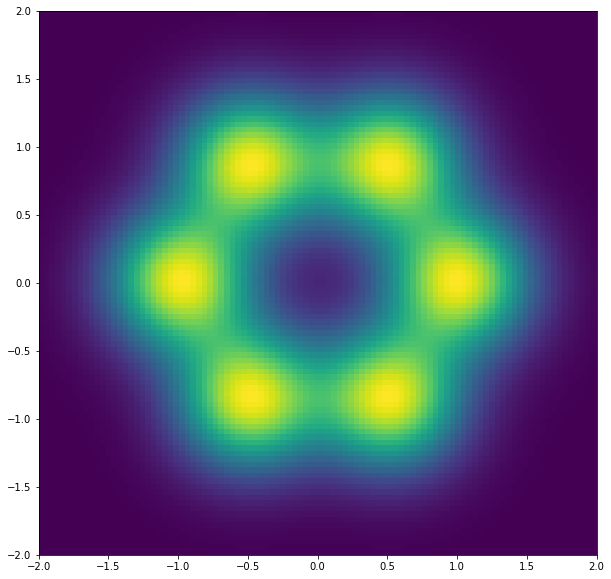}
    \end{subfigure}
    \begin{subfigure}[t]{0.32\linewidth}
    \includegraphics[trim={1.2cm 0.8cm 0.0cm 0.0cm},clip,width = \linewidth]{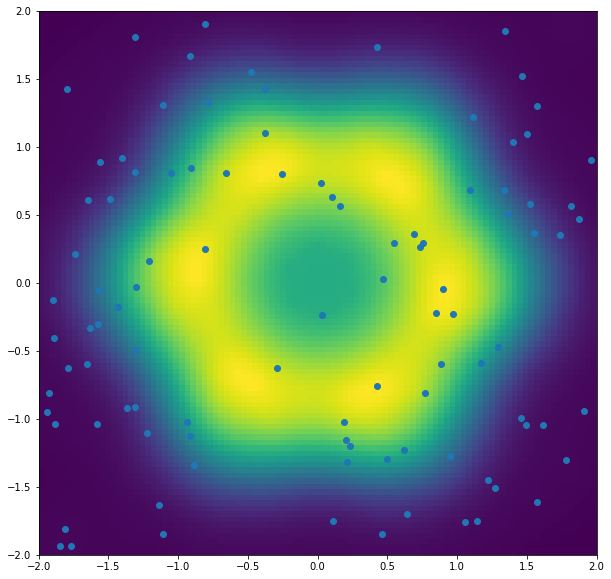}
    \end{subfigure}
    \caption{Six-fold symmetric test function (center) and GP approximations. The image on the left shows the standard GP posterior mean using the squared-exponential kernel. On the right, we see the GP approximation using kernel \eqref{eq:six-fold}. All points are implicitly used six times increasing the amount of information used for the regression six-fold.}
    \label{fig:six-fold}
\end{figure}
\subsection{The Essence of Stationary Kernels vs Non-Stationary Kernels}
The word kernel stems from the theory of integral operators. See reference \cite{williams2006gaussian} for more explanation on the origin of kernels and the  connection to integral operators.
Stationary kernels are of the form
\begin{equation}\label{eq:stat}
    k=k(||\mathbf{x}_1-\mathbf{x}_2||),
\end{equation}
i.e. they are function of a norm placed on $\mathcal{X}$.
Correlation of data that only depend on the distance of data points and not on the respective locations is referred to as
``stationary''.

\vspace*{2mm}
\noindent
In contrast, 
non-stationary kernels are more general p.s.d. functions of the form
\begin{equation}\label{eq:nonstat}
    k=k(\mathbf{x}_1,\mathbf{x}_2) ,
\end{equation}
where the kernel function now contains as arguments the location of the data points, and hence relaxes the restriction so that 
$k=k(\mathbf{x}_1,\mathbf{x}_2) \neq k(||\mathbf{x}_1-\mathbf{x}_2||)$.

\vspace*{2mm}
\noindent
Stationary and non-stationary kernels are both symmetric positive-definite functions, since they induce inner products in $\mathcal{H}$.
The difference between stationary and non-stationary kernels can be illustrated visually by a one-dimensional example.
If we consider $\mathcal{X}~\subset~\mathbb{R}^1$ and therefore $f=f(x)$, we can illustrate
the kernel as a function over $\mathcal{X}~\times~\mathcal{X}$ (see Figure \ref{fig:statvsnon-stat}).  Stationary
kernel functions are constant along diagonals, unlike non-stationary kernels.
This characteristic of non-stationary kernels translates into potentially highly flexible inner products, and therefore similarity measures, which can encode covariances within the input space, the output space and across the two spaces. 
\begin{figure}
    \centering
    \begin{subfigure}[t]{0.45\linewidth}
    \includegraphics[trim={0cm 0cm 0.0cm 0.0cm},clip,width = \linewidth]{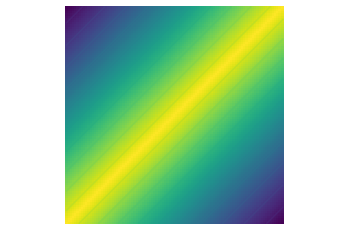}
    \end{subfigure}    
    \begin{subfigure}[t]{0.45\linewidth}
    \includegraphics[trim={0cm 0cm 0.0cm 0.0cm},clip,width = \linewidth]{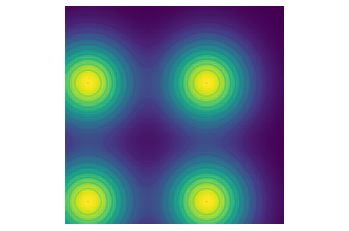}
    \end{subfigure}
    \caption{A top view onto a stationary (left) and non-stationary (right) kernel function over $\mathcal{X}~\times~\mathcal{X}$. 
    Both functions are symmetric; however,
    while the stationary kernel function
    is constant along the diagonals, the non-stationary kernel function has no such 
    restriction. Non-stationary kernels are therefore a much more flexible inner product in $\mathcal{H}$, which translates into a more flexible similarity measure. This added flexibility can be used to communicate information into remote corners of the index set.}
    \label{fig:statvsnon-stat}
\end{figure}
\subsection{Advanced Non-Stationary Kernel Designs}
Non-stationary kernels have the additional flexibility that they depend on the location of the input points, not only on the distance between them (see Equation \eqref{eq:nonstat}).
This gives a learning algorithm powerful additional capabilities since the similarity measure between data can vary substantially across $\mathcal{X}~\times~\mathcal{X}$ (see Figure \ref{fig:statvsnon-stat}).
Stationarity is an approximation that almost never holds in real data sets. Imagine a regression model of the topography of the United States.
While correlation lengths in the Sierra Nevada and in the Rocky Mountains are in the order of miles, they will be hundreds of miles in the Great Plains. Using a stationary kernel would perform poorly in such a scenario. Non-stationary kernels, on the other hand, can capture the varying length scales and lead to accurate function approximations. 
This extra flexibility is also useful for multi-output Gaussian processes in which distances between tasks are arbitrary and any stationary choice would limit the method's ability to learn. In fact, we will see how flexible non-stationary kernels can replace
tailored methods for multi-output GPR.

\vspace*{2mm}
\noindent
When designing advanced non-stationary kernels, we have to show that the resulting kernel functions are positive semi-definite, just like in the stationary case. 
However, it is often difficult to prove positive semi-definiteness in closed form for general functions.
Instead, one can induce positive semi-definiteness by taking advantage of the fact that, as mentioned earlier, the set of kernels is closed under addition, multiplication, and linear transformations. In addition, we can show that:

\vspace*{2mm}
\noindent
\begin{theorem}\label{theorem:ff}
Let $k(\mathbf{x}_1,\mathbf{x}_2)$ be a valid kernel, then $f(\mathbf{x}_1)f(\mathbf{x}_2)k(\mathbf{x}_1,\mathbf{x}_2)$
is also a valid kernel according to Definition \ref{def:kernel}. Here, $f(x)$ is an arbitrary function.
\end{theorem}
\begin{proof}
$\sum_i^N\sum_j^N~ b_i~b_j~k(\mathbf{x}_i,\mathbf{x}_j)~\geq~0 ~\forall N,~\mathbf{x}\in~\mathbb{R}^N, ~\mathbf{b}\in~\mathbb{R}^N \\
\Rightarrow \sum_i^N\sum_j^N~ c_i~c_j~f(\mathbf{x}_i)~f(\mathbf{x}_j)~k(\mathbf{x}_i,\mathbf{x}_j)~\geq~0;~\forall\mathbf{c}\in~\mathbb{R}^N \\
\Rightarrow \sum_i^N\sum_j^N~f(\mathbf{x}_i)~f(\mathbf{x}_j)~k(\mathbf{x}_i,\mathbf{x}_j)~\geq~0~\forall N,~\mathbf{x}\in~\mathbb{R}^N$
\end{proof}

\vspace*{2mm}
\noindent
Since a constant is a valid kernel, $f(\mathbf{x}_1)f(\mathbf{x}_2)$ has to be a valid kernel too.

\vspace{2mm}
\noindent
The kernel $f(\mathbf{x}_1)f(\mathbf{x}_2)k(\mathbf{x}_1,\mathbf{x}_2)$ represents a trade-off between flexibility and simplicity and is used in our examples. Figure \ref{fig:non-stat_1d} shows that particular kernel for a linear function $f$ in one dimension, and underscores how using a stationary kernel leads to
under-estimated and over-estimated posterior variances when the length scales vary across $\mathcal{X}$.
In the example seen in Figure \ref{fig:non-stat2d}, we are using the kernel defined in Theorem \ref{theorem:ff} with
\begin{align}
f(\mathbf{x}_i)~=~(\phi_1~(\sqrt{50}-||\mathbf{x}_i||)) + \phi_2.
\end{align}
The result shows the incorrectly-estimated variances when the stationary kernel is used. 
\begin{figure}
    \centering
    \begin{subfigure}[t]{\linewidth}
    \includegraphics[trim={0cm 0cm 0.9cm 0.8cm},clip,width = \linewidth]{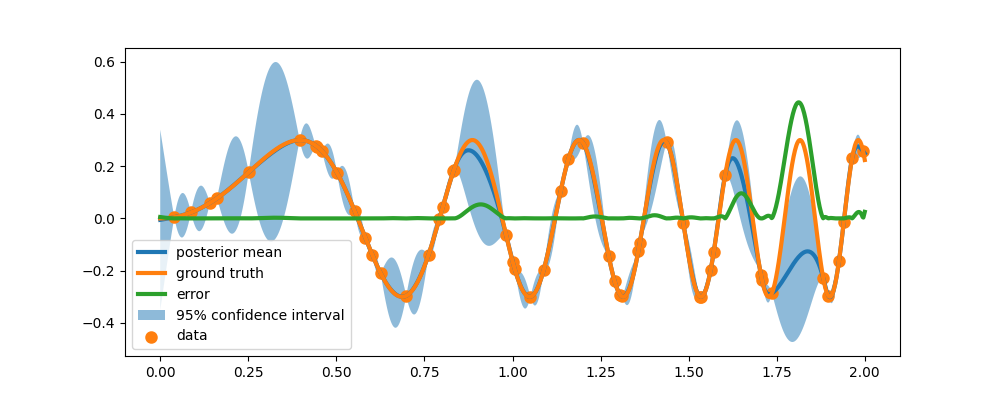}
    \end{subfigure}
    
    \begin{subfigure}[t]{\linewidth}
    \includegraphics[width = \linewidth]{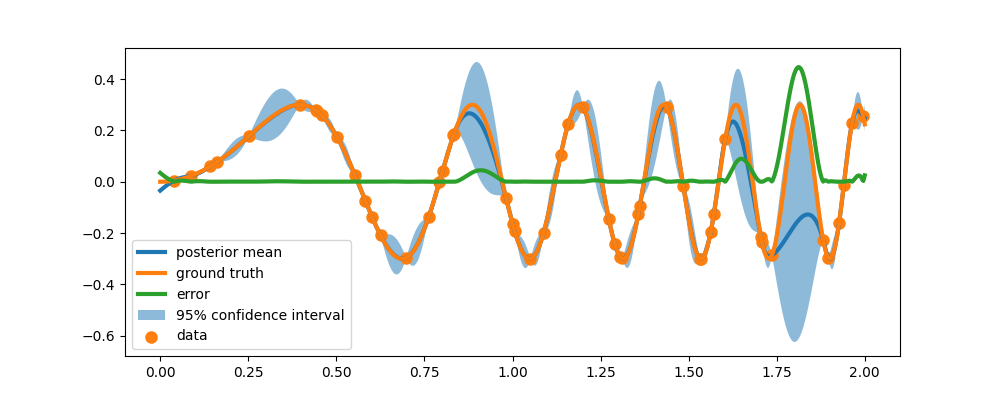}
    \end{subfigure}
    \caption{Comparison of a one-dimensional result of a Gaussian process using stationary (top) and non-stationary (bottom) kernels.
    In this figure, we use the kernel from Theorem \ref{theorem:ff} with a linear function, i.e. $k(x_1,x_2)=x_1 x_2 k_{matern}$.
    The stationary-kernel Gaussian process (top) significantly overestimates posterior variances on the left and underestimates them on the right. This is due to the fact that the similarity at a given distance is averaged across the domain. The non-stationary-kernel Gaussian process
    uses the location to compute the similarity and can therefore calculate the posterior variances more accurately.}
    \label{fig:non-stat_1d}
\end{figure}
\begin{figure}
    \centering
    \begin{subfigure}[t]{0.24\linewidth}
    \includegraphics[width = \linewidth]{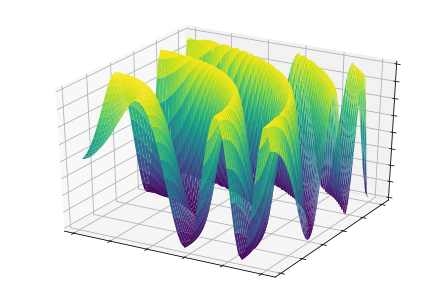}
    \end{subfigure}    
    \begin{subfigure}[t]{0.24\linewidth}
    \includegraphics[width = \linewidth]{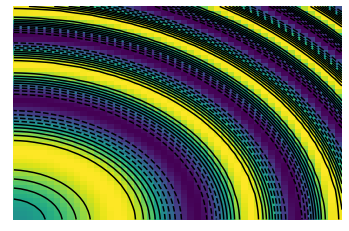}
    \end{subfigure}
    \begin{subfigure}[t]{0.24\linewidth}
    \includegraphics[width = \linewidth]{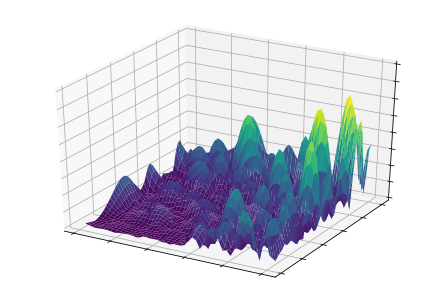}
    \end{subfigure}    
    \begin{subfigure}[t]{0.24\linewidth}
    \includegraphics[width = \linewidth]{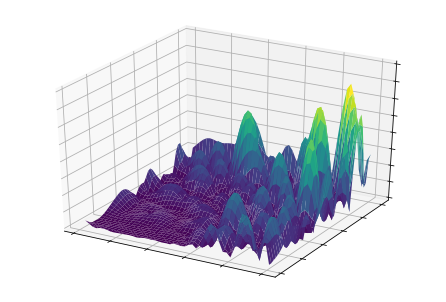}
    \end{subfigure}
    \put(-340,70){{\makebox(0,0)[lb]{ground truth}}}
    \put(-110,70){{\makebox(0,0)[lb]{$|\Delta|$}}}
    \put(-330,20){{\makebox(0,0)[lb]{$-1$}}}
    \put(-320,50){{\makebox(0,0)[lb]{$1$}}}
    \put(-115,20){{\makebox(0,0)[lb]{$0$}}}
    \put(-115,50){{\makebox(0,0)[lb]{$2$}}}
    \put(-10,20){{\makebox(0,0)[lb]{$0$}}}
    \put(-10,50){{\makebox(0,0)[lb]{$2$}}}
    
    \vspace{4mm}
    \begin{subfigure}[t]{0.24\linewidth}
    \includegraphics[width = \linewidth]{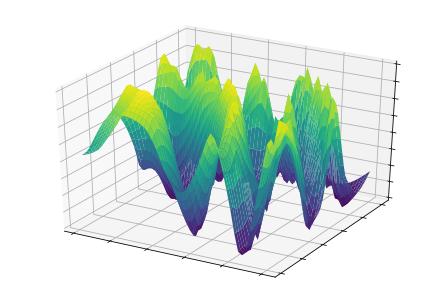}
    \end{subfigure}    
    \begin{subfigure}[t]{0.24\linewidth}
    \includegraphics[width = \linewidth]{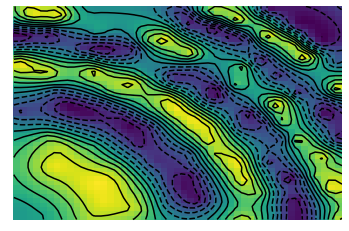}
    \end{subfigure}
    \begin{subfigure}[t]{0.24\linewidth}
    \includegraphics[width = \linewidth]{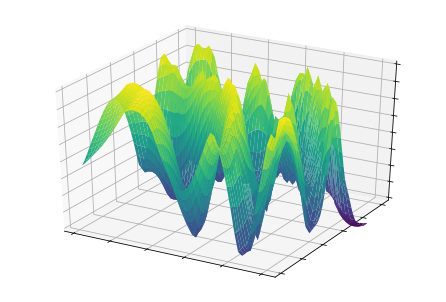}
    \end{subfigure}    
    \begin{subfigure}[t]{0.24\linewidth}
    \includegraphics[width = \linewidth]{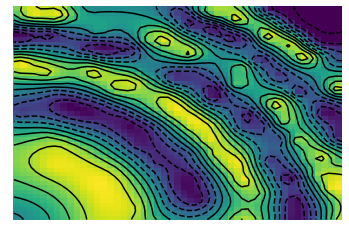}
    \end{subfigure}
    \put(-420,10){\rotatebox{90}{\makebox(0,0)[lb]{posterior mean}}}
    \put(-330,70){{\makebox(0,0)[lb]{\small stationary kernel}}}
    \put(-140,70){{\makebox(0,0)[lb]{\small non-stationary kernel}}}
    \put(-330,20){{\makebox(0,0)[lb]{$-1$}}}
    \put(-320,50){{\makebox(0,0)[lb]{$1$}}}
    \put(-115,20){{\makebox(0,0)[lb]{$-1$}}}
    \put(-115,50){{\makebox(0,0)[lb]{$1$}}}
    
    \vspace{4mm}
    \begin{subfigure}[t]{0.24\linewidth}
    \includegraphics[width = \linewidth]{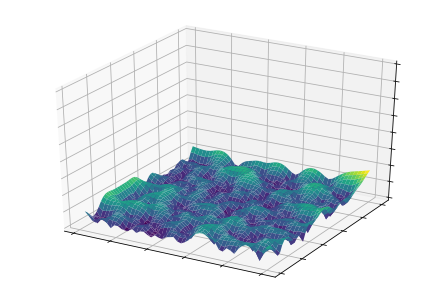}
    \end{subfigure}    
    \begin{subfigure}[t]{0.24\linewidth}
    \includegraphics[width = \linewidth]{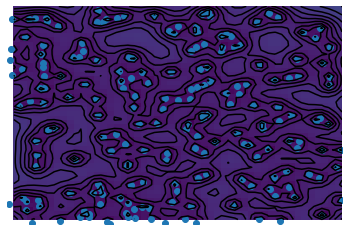}
    \end{subfigure}
    \begin{subfigure}[t]{0.24\linewidth}
    \includegraphics[width = \linewidth]{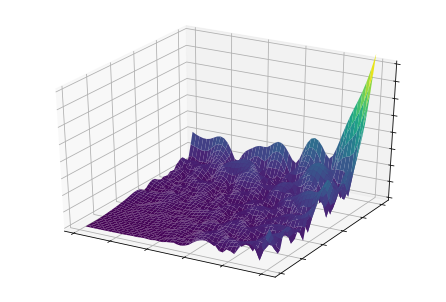}
    \end{subfigure}    
    \begin{subfigure}[t]{0.24\linewidth}
    \includegraphics[width = \linewidth]{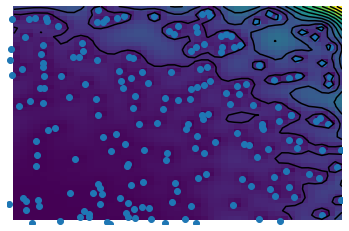}
    \end{subfigure}
    \put(-420,0){\rotatebox{90}{\makebox(0,0)[lb]{posterior variance}}}
    \put(-330,70){{\makebox(0,0)[lb]{\small stationary kernel}}}
    \put(-140,70){{\makebox(0,0)[lb]{\small non-stationary kernel}}}
    \put(-325,20){{\makebox(0,0)[lb]{$0$}}}
    \put(-320,50){{\makebox(0,0)[lb]{$2$}}}
    \put(-115,20){{\makebox(0,0)[lb]{$0$}}}
    \put(-115,50){{\makebox(0,0)[lb]{$2$}}}
    \caption{Comparison of a two-dimensional GP approximation of a function over $\mathcal{X}=[0,5]\times [0,5]$ using stationary and non-stationary kernels. Similar to what we have seen in Figure \ref{fig:non-stat_1d}, the
    posterior variance is significantly overestimated close to the origin where the frequency of the function is low. 
    The farther we move away from the origin, the more does the stationary-kernel GP underestimate the posterior variance. This is of significant impact, for instance, for autonomous data-acquisition where the posterior variance plays a large role in the choice of the next measurement.}
    \label{fig:non-stat2d}
\end{figure}
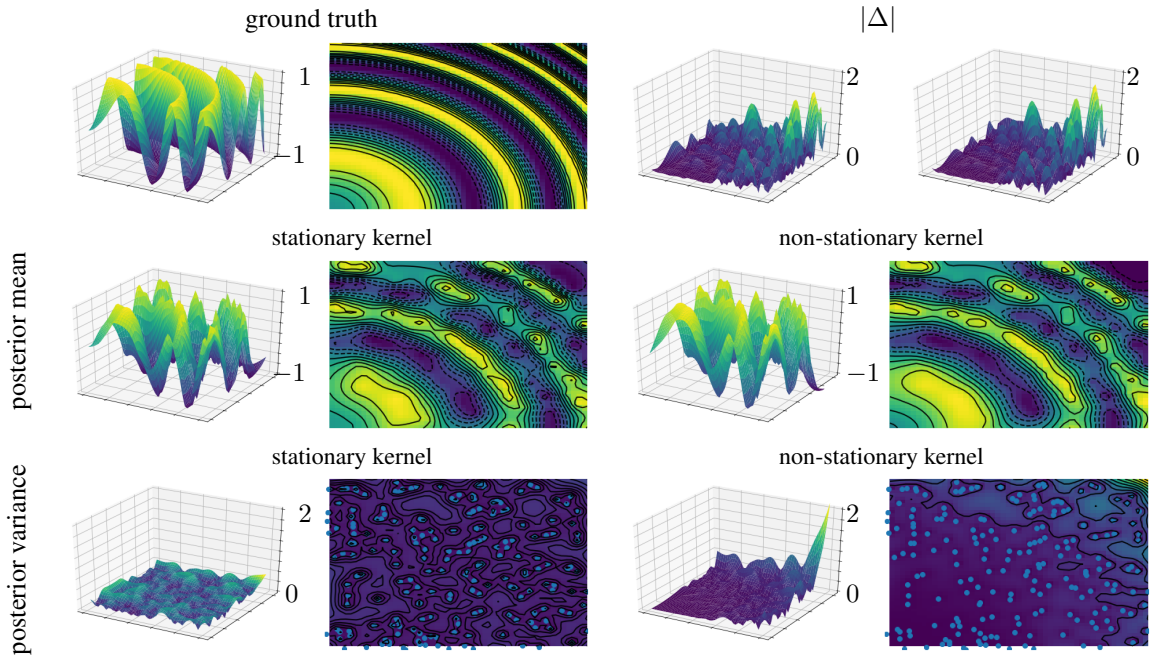

\vspace*{2mm}
\noindent
A particularly flexible kernel, introduced by \cite{paciorek2006spatial} and reformulated and enhanced by \cite{risser2015regression}, 
is defined as 
\begin{equation}\label{eq:non_stat_risser1}
    k(\mathbf{x}_1,\mathbf{x}_2)~=~ \frac{\sigma_s^2(\mathbf{x}_1)~\sigma_s^2(\mathbf{x}_2)}{\sqrt{|\frac{\pmb{\Sigma}(\mathbf{x}_1)+\pmb{\Sigma}(\mathbf{x}_2)}{2}|}} \mathcal{M}(\sqrt{Q(\mathbf{x}_1,\mathbf{x}_2)})
\end{equation}
where 
\begin{equation}\label{eq:non_stat_risser2}
    Q(\mathbf{x}_1,\mathbf{x}_2) = (\mathbf{x}_1 - \mathbf{x}_2)^T~ \Big( \frac{\pmb{\Sigma}(\mathbf{x}_1)+\pmb{\Sigma}(\mathbf{x}_2)}{2} \Big) ^{-1}~(\mathbf{x}_1 - \mathbf{x}_2),
\end{equation}
where $\mathcal{M}$ is the Mat\'ern kernel. This kernel allows for non-constant signal variances, length scales and anisotropies. In addition, it takes only a small adjustment to vary the differentiability of the model within $\mathcal{X}$. The presented non-stationary kernel designs, together with an efficient way to find the hyperparameters, renders specialized techniques for multi-task GPs obsolete. The difference between single-task and multi-task Gaussian processes can be entirely contained within the kernel design as will be discussed in the next section.
\subsection{Using Flexible Non-Stationary Kernels for Multi-Task GPR}
The challenge of multi-task GPs is that there is no natural distance between the different tasks (in $\mathcal{X}_o$). Furthermore, this distance between tasks may change depending on $\mathbf{x}_i~\in~\mathcal{X}_i$. This is why early on, several tasks would just be seen as independent Gaussian processes. Clearly, in this approach, cross-task covariances are ignored and learning can only happen within a task. To circumvent this problem, several techniques have been proposed. In the following overview, we are following the survey in \cite{borchani2015survey}, which broadly divided the multi-output regression methods into ``problem transformation methods'' and ``algorithm adaption methods''. Algorithm adaption methods are based on additional techniques and methodologies, such as support vector regression, to encode the correlation between tasks. It is able to capture cross-task correlations but, as the name suggests, needs the adaption of the basic algorithm and the theory, which often leads to new issues. Problem transformation methods, on the other hand, are based on transforming the problem into several single-task problems, creating separate models for them, and concatenating all the models. It is often criticized for not capturing the intricate correlations between the tasks. Problem transformation methods leave the basic theory of GPs intact and are therefore more general, widely applicable and avoid common problems like missing data in one or more tasks. 

\vspace*{2mm}
\noindent
We argue that the GP framework is, without alteration, able to account for the correlation between multiple tasks. For that, we draw attention to the fact that problem transformation methods in combination with flexible non-stationary kernels do not suffer from the limited ability to encode cross-task covariances.
Problem transformation methods commonly make use of separable, stationary kernels that fail to encode cross-task correlations due to arbitrary distances between tasks. We can address this issue by using flexible kernel definitions, we present in this paper.
Using advanced non-stationary kernels liberates us from the problems of multi-output Gaussian processes; we can assume a constant, arbitrary distance between the tasks
and the kernel will learn how these distances translate into similarities as a function on the index set $\mathcal{X}$.
This leaves the basic theory of GPs untouched, and is therefore
robust against common multi-task GP problems such as missing data in a subset of tasks or poor interpretability.
The advanced kernels needed for the extra flexibility come at the cost of many hyperparameters 
we have to find, which however, can be countered with clever optimization procedures.

\vspace*{2mm}
\noindent
As alluded to earlier, instead of interpolating a vector-valued function over the input space, we approximate a scalar function $f(\mathbf{x})$ on $\mathcal{X}=\mathcal{X}_i~\times~\mathcal{X}_o$.
The norm on the RKHS induces a metric which is entirely defined by the kernel and therefore extends also into the output space. Therefore, a flexible kernel overcomes the challenge of arbitrarily defined distances between tasks. In this framework, we assume $\mathcal{X}_o$ to be a subset of the indexing set  
which leads us to refer to this special kind of multi-task Gaussian processes as function-valued Gaussian processes (fvGP) --- one could imagine the outputs of an evaluation to be itself a function over $\mathbb{R}^n$.
To reiterate the key takeaway of this section, the main difference of a multi-output, or function-valued GP and a single-task GP is the choice of the kernel. The main problem with this approach is the vastly increased number of hyperparameters that have to be found. This issue will briefly be discussed in the next section.

\vspace*{2mm}
\noindent
In Figure \ref{fig:fvgp}, we defined two different tasks that show a particularly high correlation between the circled areas.
This correlation is reflected in the covariance matrix. 
The kernel for this example is defined in Equations \eqref{eq:non_stat_risser1} and \eqref{eq:non_stat_risser2} with
\begin{align}\label{eq:fvgp}
    \pmb{\Sigma}(\mathbf{x})~=&~
    \begin{pmatrix}
         l(\mathbf{x}_1,\mathbf{x}_2) & 0 & 0 \\
         0 & l(\mathbf{x}_1,\mathbf{x}_2) & 0 \\
         0 & 0 & l(\mathbf{x}_1,\mathbf{x}_2) \\
    \end{pmatrix} \nonumber \\
    l(\mathbf{x}_1,\mathbf{x}_2)~=&~h_1 + (h_2~(\exp{[(\mathbf{x}_1-\mathbf{h}_1)^T~\pmb{M}(\mathbf{x}_1-\mathbf{h}_1)]}+\exp{[(\mathbf{x}_2-\mathbf{h}_2)^T~\pmb{M}(\mathbf{x}_2-\mathbf{h}_2)]}))\nonumber  \\
\pmb{M}~=&~
\begin{pmatrix}
     h_3 & 0  & 0   \\
     0   & h_3& 0   \\
     0   & 0  & h_3 \\
\end{pmatrix},
\end{align}
where all $h_i$ and $\mathbf{h}_i$ are found by the training process. The results in Figure \ref{fig:fvgp} show that a flexible non-stationary kernel can encode complicated non-local covariances that will be used for the approximation and uncertainty quantification.
\begin{figure}
    \centering
    \includegraphics[width = \linewidth]{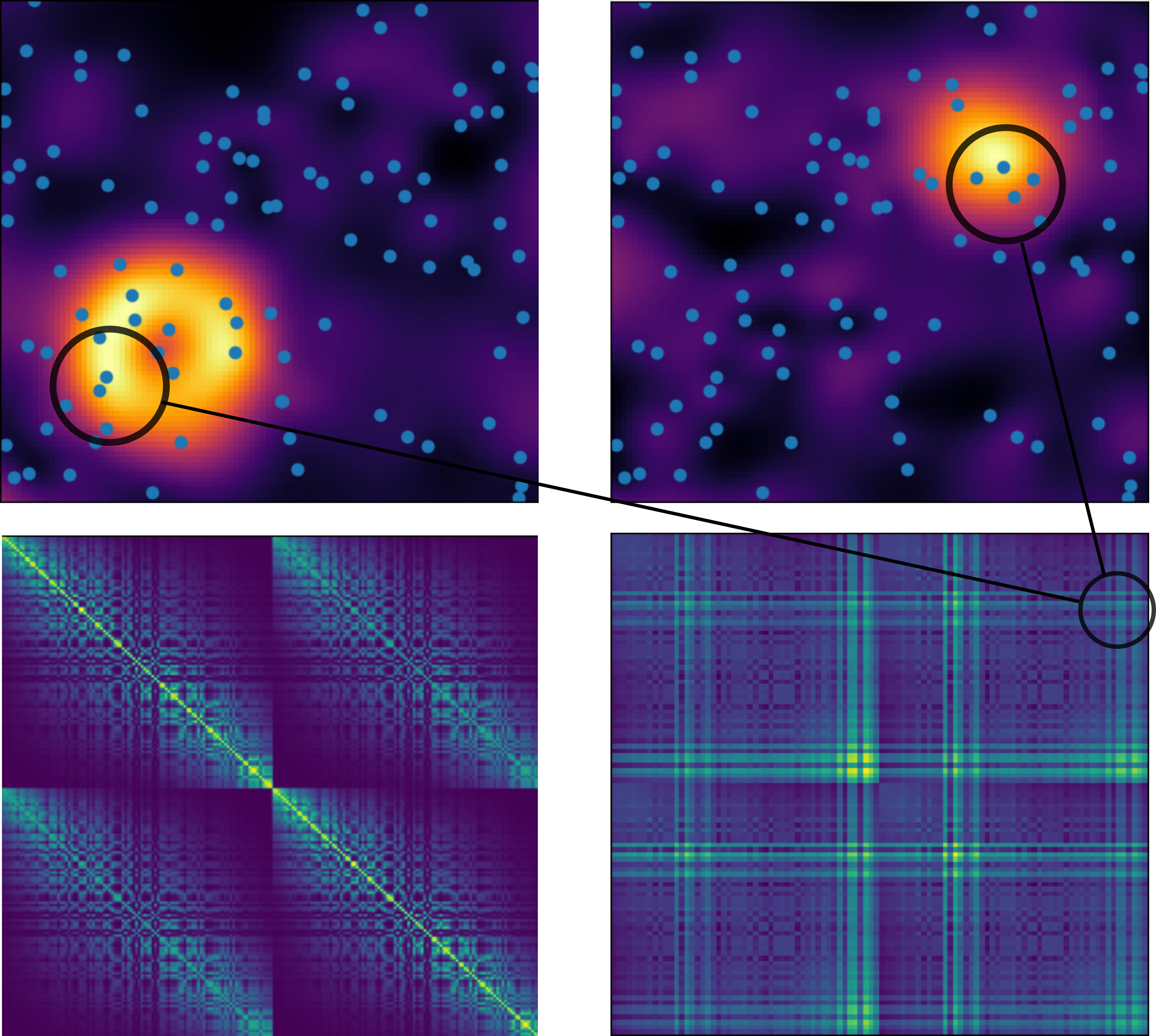}
    \caption{Representation of two tasks (top) and the associated covariance matrices (bottom) of a multi-task Gaussian process using 100 random data points. In the bottom left, the covariance is computed using a stationary kernel; while cross-covariances are not ignored (quadrant 1 and 3), they are just offset by a constant.
    The covariance in the bottom right is defined by a non-stationary kernel. This allows the Gaussian process to learn that the circled regions are correlated. The non-stationary kernel used in this example is shown in Equations \eqref{eq:non_stat_risser1} and \eqref{eq:non_stat_risser2}, using the terms defined in \eqref{eq:fvgp}. The central point here is that the fundamental difference between single-task and multi-task Gaussian processes lies in the kernel design. Note the checkerboard pattern; the length scale function is the sum of two Gaussian functions. When one point is close to one of them, the covariance with all other data points will be comparatively large. However, the covariance reaches its maximum when both points are located at the center of one or the other Gaussian function.}
    \label{fig:fvgp}
\end{figure}
\section{A Note on Optimizing the Marginal Log-Likelihood when using Advanced Kernel Designs}
As mentioned throughout this paper, the main issue that accompanies 
advanced kernel design is the number of hyperparameter we have to find. 
We can think of the hyperparameters as a vector $\pmb{\phi}~\in~\mathbb{R}^n$. 
When using standard kernel definitions, $n$ is often two or three. This number is significantly larger for
advanced stationary and especially non-stationary kernels. We have shown that we can invoke functions over $\mathcal{X}$
into the kernel definitions. These functions can be defined as the sum of arbitrary basis functions.
Their locations and coefficients are also hyperparameters and have to be found. This example makes clear that the number of hyperparameters $n$ can quickly rise to numbers that makes the marginal-log-likelihood optimization a lengthy procedure. To find the global or a high-quality local optimum,
an optimization of this scale commonly needs many function evaluations to succeed. However, each function evaluation of
the marginal-log-likelihood function is potentially costly since it involves an inversion or a system solve. 
One possible solution to the problem is the use of hybrid optimization algorithms that can run in parallel to 
the GP prediction and can 
provide best-estimate optima whenever queried. 
The log-likelihood function,
its gradient
and the Hessian 
evaluations can be accelerated using GPU computer architectures. 
In addition, we can start many local searches in parallel --- and remove the found optima by deflation --- to take full advantage of HPC computer architecture.
This is an important branch of our work and will be presented in the near future \cite{noack2017hybrid}.  
\section{Experiments}
In this section, we want to show the potential impact of the proposed kernel designs on two scientific experiments, namely neutron scattering and IR spectroscopy. The shown data has been collected at the Thales instrument at the Institute Laue-Lagevin (ILL)  in France,
and at the Berkeley Synchrotron Infrared Structural Biology (BSISB) beamline at the
Advanced Light Source (ALS) at Lawrence Berkeley National Laboratory (LBNL) in Berkeley, California. 
At these instruments, our work on 
Gaussian Processes is used for autonomous data acquisition and general analysis and interpretation purposes \cite{noack_review}.
We show how the presented improvements of kernel designs can advance the use of Gaussian processes in these experiments, and influence
the experimental design and the resulting model approximation.
\subsection{IR Spectroscopy}
Infrared (IR) imaging spectroscopy employs full infrared spectra in order to study materials and biological samples. This is done by directing an infrared beam onto the sample at a point ($x_1,x_2$). Therefore, we can define the input space 
$\mathcal{X}_i~\subset~\mathbb{R}^2$ with outputs composed of entire spectra at selected points in the input set, i.e., 
$\mathcal{X}_o~\subset~\mathbb{R}_+^1$. As before, the final index set is defined by 
$\mathcal{X}=\mathcal{X}_i \times \mathcal{X}_o~\subset~\mathbb{R}^3$. We will assume that a spectrum is represented by 87 intensity values
at a set of wave numbers. 
This example illustrates the duality of a multi-task GP over $\mathcal{X}_i$ and a single-task GP over $\mathcal{X}_i~\times~\mathcal{X}_o$, and that the difference can be contained within the used kernel. Here, the ``tasks'' have a natural distance between them, with the unit $\text{cm}^{-1}$ of a wave number.
This is not always the case. Instead of spectra, we could approximate the PCA components of spectra, 
which do not have a natural distance. 
The stationary kernel, we are using for this example, is defined as
\begin{align}
    & k(\mathbf{x}_1,\mathbf{x}_2)~=~k(|\mathbf{x}_1-\mathbf{x}_2|) = \nonumber \\
    & \sigma^2~k_{exp}\Big(|\begin{bmatrix} x_1^1 \\ x_1^2 \end{bmatrix} - \begin{bmatrix} x_2^1 \\ x_2^2 \end{bmatrix}|\Big)~
    M(|x_1^3-x_2^3|),
    \label{eq:ir_stat}
\end{align}
where $k_{exp}$ is the exponential kernel and $M$ is a Mat\'ern kernel. The Euclidean distance in the exponential kernel
is anisotropic.
The non-stationary kernel is defined by
\begin{align}
    & k(\mathbf{x}_1,\mathbf{x}_2)~= \nonumber \\
    & \sigma^2~(k_{exp}\Big(|\begin{bmatrix} x_1^1 \\ x_1^2 \end{bmatrix} - \begin{bmatrix} x_2^1 \\ x_2^2 \end{bmatrix}|,\phi_7,\phi_8 \Big)~
    M(|x_1^3-x_2^3|, \phi_{9})+A_1~A_2), \nonumber \\
    &A_1 = \exp{[(x_1^3-p_1(\mathbf{x}))^2]/\phi_6} + \exp{[(x_1^3-p_2(\mathbf{x}))^2]/\phi_6} \nonumber \\
    &A_2 = \exp{[(x_2^3-p_1(\mathbf{x}))^2]/\phi_6} + \exp{[(x_2^3-p_2(\mathbf{x}))^2]/\phi_6} \nonumber \\
    &p_1(\mathbf{x}) = \phi_0~(\phi_1~x^1)+(\phi_2~x^2) \nonumber \\
    &p_2(\mathbf{x}) = \phi_3~(\phi_4~x^1)+(\phi_5~x^2),
    \label{eq:ir_non_stat}
\end{align}
where $\phi$ is a set of hyperparameters. The focus in this expression is on $A$, which allows the covariance to depend on two
Gaussian functions that can change position in $\mathcal{X}_o$ as a linear function in $\mathcal{X}_i$. 
Figure \ref{fig:ir_nonstat} shows that the Gaussian process will take advantage of the given additional flexibility of non-stationary kernels and thereby lower the overall approximation error.

\begin{figure}
    \centering
    \begin{subfigure}[t]{0.4\linewidth}
    \includegraphics[trim = 3.65cm 1.35cm 3cm 1cm,clip, width = \linewidth]{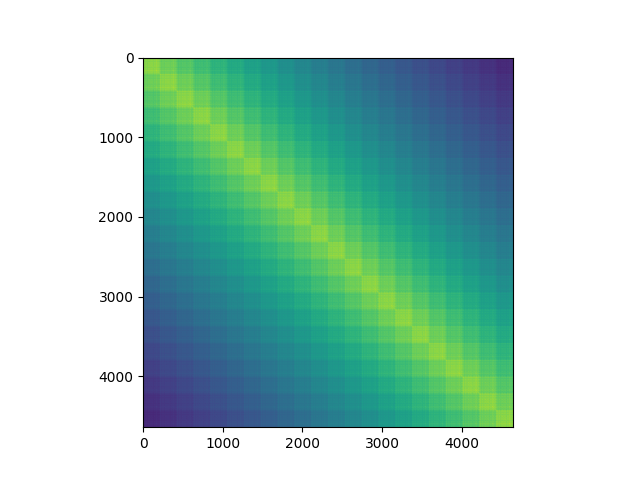}
    \caption{}
    \end{subfigure} 
    ~~~~~
    \begin{subfigure}[t]{0.4\linewidth}
    \includegraphics[trim = 3.65cm 1.35cm 3cm 1cm,clip, width = \linewidth]{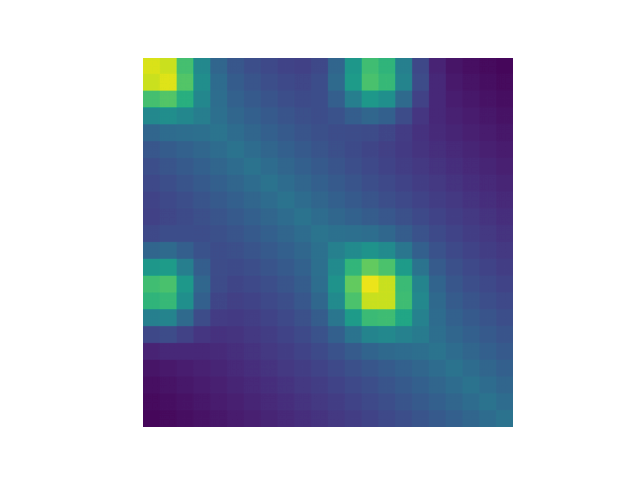}
    \caption{}
    \end{subfigure}
    
    \begin{subfigure}[t]{\linewidth}
    \includegraphics[width = \linewidth,height = 6cm]{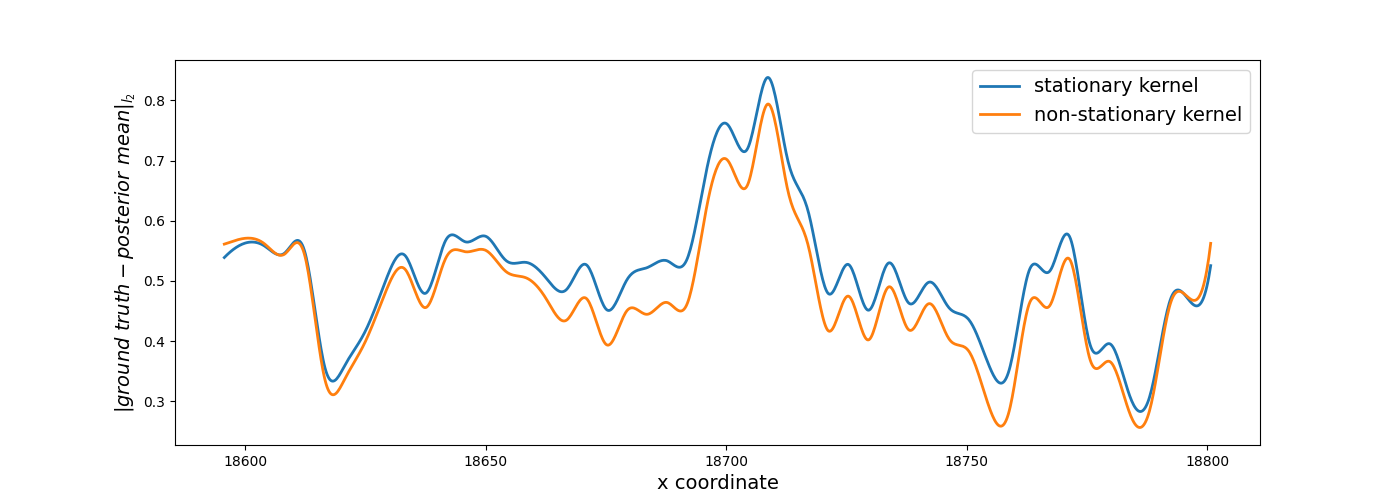}
    \caption{}
    \end{subfigure}
    
    \begin{subfigure}[t]{\linewidth}
    \includegraphics[width = \linewidth,height = 6cm]{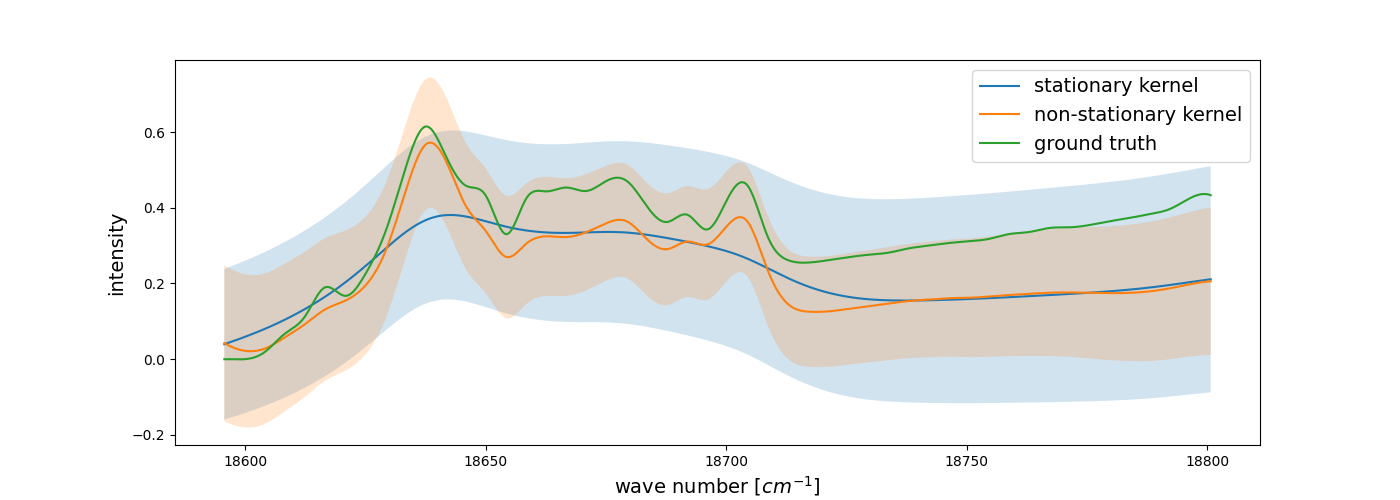}
    \caption{}
    \end{subfigure}
    \caption{Presentation of a GP used on infrared spectroscopy data deploying stationary and non-stationary kernels. (top) Simplified view on the covariance matrices resulting from stationary (a) and non-stationary (b) kernel definitions (Equations \eqref{eq:ir_stat} and \eqref{eq:ir_non_stat} respectively). The covariance matrix resulting form the stationary kernel (a), cannot identify differing similarities between tasks when their distance is constant; therefore we see a diagonal pattern of the covariance matrix. In (b) we see the covariance matrix resulting from a non-stationary kernel which is able to identify similarities between any two tasks independently. Tasks in this case can be understood as the spectrum intensity at a particular wave number. (c) The mean of the Euclidean distance between the posterior means and ground truth. The GP using the non-stationary kernel performs significantly better. (d) Posterior means and the ground truth of a representative spectrum. Not only is the approximation using the non-stationary kernel significantly more accurate, the posterior variance is overall smaller and more detailed.}
    \label{fig:ir_nonstat}
\end{figure}

\subsection{Neutron Scattering}
Neutron scattering is an experimental technique to obtain detailed information about the arrangements of atoms in condensed matter. 
The data showcasing symmetry comes from 
the Thales (Three Axis Low Energy Spectrometer) at ILL \cite{weber2019polarized}.
The measurements probe a function $S(q_h,q_l,q_k,E)$ which is often symmetric around one or more axes. Figure \ref{fig:ns_symm}
shows how effectively kernels that impose symmetry \eqref{eq:symm_kernel} can be used to approximate the model function and to steer data acquisition in the presence of symmetry. The higher-quality approximation translates into less points that are needed for a targeted model accuracy. 
\begin{figure}
    \begin{subfigure}[t]{0.48\linewidth}
    \includegraphics[height = 5cm]{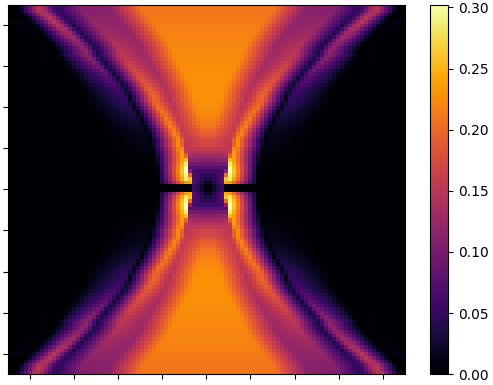}
    \caption{}
    \end{subfigure}    
    \begin{subfigure}[t]{0.48\linewidth}
    \centering
    \includegraphics[height = 5cm]{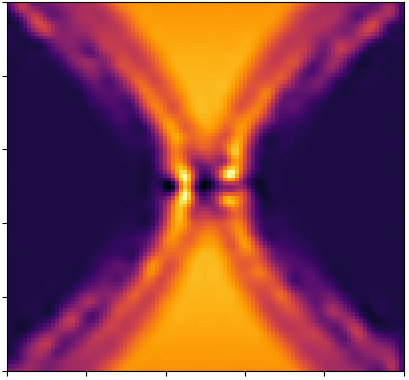}
    \caption{}
    \end{subfigure}
    \put(-420,50){\rotatebox{90}{\makebox(0,0)[lb]{$Energy$}}}
    \put(-340,-5){{\makebox(0,0)[lb]{$q_h$}}}
    \put(-105,-5){{\makebox(0,0)[lb]{$q_h$}}}
    
    \begin{subfigure}[t]{0.48\linewidth}
    \includegraphics[height = 5cm]{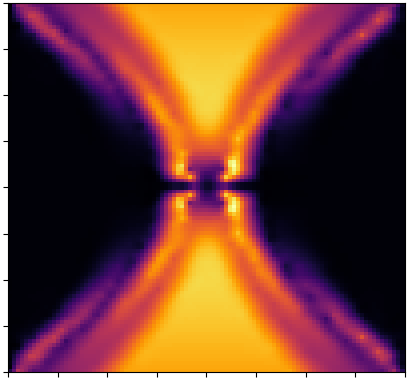}
    \caption{}
    \end{subfigure}
    \begin{subfigure}[t]{0.48\linewidth}
    \centering
    \includegraphics[width = \linewidth]{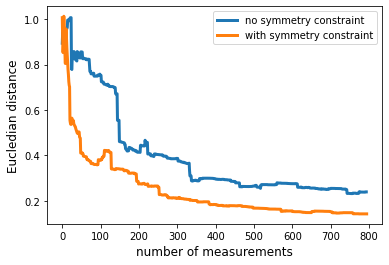}
    \caption{}
    \end{subfigure} 
    \put(-420,50){\rotatebox{90}{\makebox(0,0)[lb]{$Energy$}}}
    \put(-340,-10){{\makebox(0,0)[lb]{$q_h$}}}
    \caption{Figure displaying the importance of symmetry in neutron scattering data acquisition. The function $S=S(q_h,q_l,q_k,E)$ is symmetric around $E~=~0$ for the slice of interest. Exploiting this fact as a constraint, which is enforced by advanced kernel design, increases the accuracy of the GP interpolation significantly. (a)  The ground truth function $S(q_h,q_l=0,q_k=0,E)$.  (b) The GP posterior mean when a standard exponential kernel is used. (c) The GP posterior mean with a kernel that enforces symmetry. (d) The Euclidean error of both approximations after a number of measurements.}
    \label{fig:ns_symm}
\end{figure}
\section{Discussion and Conclusion}
In this paper, we presented some known and new kernel designs which are
of interest for practitioners using GPs.
The presented kernels are able to significantly reduce uncertainty of the model,
given a number of data points. This was either achieved 
by using stationary kernels that implicitly subject the posterior mean to hard constraints, such as periodicity or symmetry, or by
non-stationary kernels that are able to encode flexible inner products which translates into the ability to
learn more complicated covariances across the input space $\mathcal{X}$. This led to a flexible and powerful formulation of multi-task Gaussian processes. 

\vspace{2mm}
\noindent
Using the appropriate kernels,
knowledge that the model function is additive can result in an immensely powerful
Gaussian process in which information from the sides of the domain can be propagated to infinity (Fig. \ref{fig:additive}).
Multiplicative kernels have no such property, but are perfectly suited for allowing axial anisotropy in $\mathcal{X}$.

\vspace{2mm}
\noindent
Figure \ref{fig:symmetric_ackleys} showed how the quality of the posterior mean can increase when
symmetry is present and accounted for by the kernel. The same was shown for experimental data in 
Figure \ref{fig:ns_symm}. Symmetry implicitly increases the amount of information that is used in the prediction.
For instance, in case of axial symmetry around the $x$ and $y$ axes, every data point contains four times the amount of information -- compared to the use of a standard kernel --- decreasing the computational cost for a given function-approximation problem by 64 (assuming $O(N^3)$ scaling). We have shown that periodicity can be accounted for in the same manner. Note that the imposed periodicity
is not the same as a sine kernel, since imposing periodicity does not impose any particular functional shape.

\vspace{2mm}
\noindent
Lastly, we investigated and drew attention to non-stationary kernels and their impact on a Gaussian process. 
We have seen that a GP with a constant length scale is prone to both
over- and underestimating the posterior variance (Figures \ref{fig:non-stat_1d} and \ref{fig:non-stat2d}) which has major 
implications for decision-making algorithms which use the posterior covariance, e.g., for optimal experiment design.
Non-stationary kernels can also be used to obtain a flexible and simple implementation of multi-output GPs.
In an experimental setting, the non-stationary kernels led to an overall high-quality approximation (Fig. \ref{fig:ir_nonstat}).
We have rediscovered that the main difference between single and multi-task GPs can be entirely contained within the kernel design, which
leaves the basic theory of GPs untouched.

\vspace{2mm}
\noindent
The use of advanced kernel designs does not come without issues, especially in the multi-task setting.
The biggest issue is the added computational costs. Clearly, the marginal log-likelihood optimization
becomes a much more involved process as the number of hyperparameters increases. But there are ways to overcome this shortcoming.
Traditional Gaussian-process training uses standard optimization procedures to find the hyperparameters, such as multi-start gradient descent.
When the optimization has to find more hyperparameters, more sophisticated, HPC-ready algorithms have to be used. In an optimal and sequential design setting,
the optimization can happen asynchronously, so that the costs of the training are hidden. Such an algorithm is being developed by the authors.  


\section*{Acknowledgment}
The work was funded through the Center for Advanced
Mathematics for Energy Research Applications 
(CAMERA), which is jointly funded by the 
Advanced Scientific Computing Research (ASCR) 
and Basic Energy Sciences (BES) within the 
Department of Energy's Office of Science, under Contract No. DE-AC02-05CH11231. 
The data was provided by the BSISB program (DOE No. DE-AC02-05CH11231) and the Thales project at ILL, France.
We want to thank the groups for the collaboration and data. 

\section*{Author Contribution}
M.M.N developed the mathematics and wrote the first draft of the manuscript. J.A.S. supervised the work, verified the correctness of the mathematical derivations and revised the manuscript.

\clearpage
\bibliographystyle{plainnat}
\bibliography{literature}
\end{document}